\documentclass{article}

\usepackage{PRIMEarxiv}

\usepackage[utf8]{inputenc} 
\usepackage[T1]{fontenc}    
\usepackage{hyperref}       
\usepackage{url}            
\usepackage{booktabs}       
\usepackage{amsfonts}       
\usepackage{nicefrac}       
\usepackage{microtype}      
\usepackage{lipsum}
\usepackage{fancyhdr}       
\usepackage{graphicx}       
\graphicspath{{media/}}     
\usepackage{graphicx}
\usepackage{subcaption}
\usepackage{multirow}
\usepackage{url}
\usepackage{amsmath}
\usepackage{amssymb}
\usepackage{mathtools}
\usepackage{amsthm}
\usepackage{array}
\usepackage{natbib}
\usepackage{wrapfig}

\theoremstyle{plain}
\newtheorem{theorem}{Theorem}[section]

\newtheorem{lemma}[theorem]{Lemma}

\theoremstyle{definition}

\newtheorem{assumption}[theorem]{Assumption}
\theoremstyle{remark}
\newtheorem{remark}[theorem]{Remark}

\usepackage{amsmath,amsfonts,bm}

\def\eqref#1{equation~\ref{#1}}

\def\1{\bm{1}}

\DeclareMathAlphabet{\mathsfit}{\encodingdefault}{\sfdefault}{m}{sl}
\SetMathAlphabet{\mathsfit}{bold}{\encodingdefault}{\sfdefault}{bx}{n}

\newcommand{\bsA}{{\boldsymbol{A}}}
\newcommand{\bsW}{{\boldsymbol{W}}}
\newcommand{\bsa}{{\boldsymbol{a}}}
\newcommand{\bsb}{{\boldsymbol{b}}}
\newcommand{\bsc}{{\boldsymbol{c}}}
\newcommand{\bsh}{{\boldsymbol{h}}}
\newcommand{\bsx}{{\boldsymbol{x}}}
\newcommand{\bsy}{{\boldsymbol{y}}}
\newcommand{\bsz}{{\boldsymbol{z}}}
\newcommand{\bseta}{{\boldsymbol{\eta}}}
\newcommand{\bsxi}{{\boldsymbol{\xi}}}
\newcommand{\bsmu}{{\boldsymbol{\mu}}}
\newcommand{\bsphi}{{\boldsymbol{\phi}}}

\title{From Sparse Sensors to Continuous Fields:\\ STRIDE for Spatiotemporal Reconstruction
}

\author{
  Yanjie Tong, Peng Chen \\
  School of Computational Science and Engineering \\
  Georgia Institute of Technology \\
  Atlanta, GA 30332\\
  \texttt{\{ytong80, pchen402\}@gatech.edu} \\
}

\begin{document}
\maketitle

\begin{abstract}
Reconstructing high-dimensional spatiotemporal fields from sparse point-sensor measurements is a central challenge in learning parametric PDE dynamics. Existing approaches often struggle to generalize across trajectories and parameter settings, or rely on discretization-tied decoders that do not naturally transfer across meshes and resolutions. We propose STRIDE (Spatio-Temporal Recurrent Implicit DEcoder), a two-stage framework that maps a short window of sensor measurements to a latent state with a temporal encoder and reconstructs the field at arbitrary query locations with a modulated implicit neural representation (INR) decoder. Using the Fourier Multi-Component and Multi-Layer Neural Network (FMMNN) as the INR backbone improves representation of complex spatial fields and yields more stable optimization than sine-based INRs. We provide a conditional theoretical justification: under stable delay observability of point measurements on a low-dimensional parametric invariant set, the reconstruction operator factors through a finite-dimensional embedding, making STRIDE-type architectures natural approximators. Experiments on four challenging benchmarks spanning chaotic dynamics and wave propagation show that STRIDE outperforms strong baselines under extremely sparse sensing, supports super-resolution, and remains robust to noise.
\end{abstract}

\section{Introduction}
Many scientific and engineering systems are governed by nonlinear partial differential equations (PDEs) and evolve on high-dimensional spatiotemporal fields that depend on parameters (e.g., forcing, geometry, material properties, or initial/boundary conditions) \cite{brunton_machine_2020, wang_physics-guided_2023}. In practice, however, measurements are often restricted to sparse point sensors, so recovering the full field amounts to an ill-posed state estimation problem \cite{erichson_shallow_2020}. Although high-fidelity solvers can produce accurate trajectories, their computational cost limits repeated simulation across parameter settings and real-time deployment \cite{karniadakis_physics-informed_2021}. These constraints motivate learned surrogates that reconstruct full spatiotemporal states from sparse measurements while generalizing across trajectories and (possibly unobserved) parameters.

Existing data-driven approaches for mapping sparse observations to full states fall broadly into two categories: methods that reconstruct from concurrent measurements only \cite{santos_development_2023, luo_continuous_2024}, and methods that explicitly exploit temporal history (e.g., delay embeddings or recurrent state estimators) \cite{williams_sensing_2024, tomasetto_reduced_2025, lyu_wavecastnet_2024, song_forecasting_2024}. The former can perform well when the mapping is nearly instantaneous, but can struggle in parametric or chaotic regimes where the same sensor snapshot may correspond to multiple underlying states. The latter reduce this ambiguity by using a short observation window, echoing classical embedding and observability ideas \cite{takens1981detecting, sauer1991embedology}, but many existing designs still rely on decoders tied to a fixed discretization or whose parameter count scales with the output dimension, limiting transfer to irregular meshes and varying resolutions.

A complementary approach is to represent fields with \emph{implicit neural representations} (INRs), i.e., continuous coordinate-to-field maps that can be queried at arbitrary locations, naturally supporting irregular meshes and super-resolution. Among INR backbones, SIREN has demonstrated strong expressivity via sine activations \cite{sitzmann_implicit_2020}. In PDE learning, modulated INRs (e.g., FiLM-modulated SIREN) have been used in operator learning and spatiotemporal prediction \cite{perez_film_2017, serrano_operator_2023, du_confild_2024}. While sine-activated networks like SIREN are expressive, they often encounter challenging optimization landscapes when modeling high-frequency content. To address this, the Fourier Multi-Component and Multi-Layer Neural Network (FMMNN) was recently introduced to enhance optimization stability and improve the representation of fine-scale spatial structures \cite{zhang_fourier_2025}.

We propose STRIDE (Spatio-Temporal Recurrent Implicit DEcoder), a two-stage framework that combines windowed sensor history with continuous decoding. STRIDE maps a short observation window to a low-dimensional latent state with a temporal encoder (e.g., LSTM \cite{hochreiter_long_1997}, or alternatives such as Mamba \cite{gu_mamba_2024}) and reconstructs the field at arbitrary query locations with a modulated INR decoder. Our default instantiation uses FMMNN with lightweight shift modulations, enabling resolution-invariant reconstruction while improving training stability relative to sine-based backbones. We further provide a conditional theoretical justification: under stable delay observability of point measurements on an effectively low-dimensional parametric invariant set (Ma\~n\'e-type projection results \cite{mane1981dimension}), the reconstruction operator factors through a finite-dimensional embedding, making STRIDE-type architectures natural approximators.

\textbf{Our contributions:}

\begin{itemize}
    \item \textbf{Two-stage sparse sensors to continuous field reconstruction.} We introduce STRIDE, which combines a window encoder (history/delay information) with a modulated INR decoder to reconstruct continuous spatiotemporal fields from sparse point sensors in multi-trajectory, parametric PDE settings, without assuming the parameters are observed at inference.

    \item \textbf{Discretization- and resolution-invariant decoding.} By decoding through a coordinate-based INR and training with randomized spatial sampling, STRIDE applies to irregular meshes and supports super-resolution by querying the decoder at arbitrary locations.

    \item \textbf{Stable, powerful INR decoding via FMMNN.} We develop an INR decoder design based on FMMNN with shift modulations, yielding improved representation of spatial fields and favorable training dynamics compared to modulated SIREN backbones.

    \item \textbf{Conditional theoretical justification.} Under a stable delay-observability assumption on an effectively low-dimensional parametric invariant set, we formalize the existence and stability of the reconstruction operator and show that it factors through a finite-dimensional embedding, supporting STRIDE-type approximations.

    \item \textbf{Extensive empirical evaluation.} We benchmark STRIDE on four challenging datasets (chaotic dynamics, fluid flow, shallow water, and seismic wave propagation) and provide ablations over sensor budgets, window length, noise, spatial resolutions, and architectural choices (temporal encoders and INR backbones).
\end{itemize}

Paper organization: Section~\ref{sec:method} presents STRIDE, Section~\ref{sec:theory} provides the theoretical justification based on delay observability, Section~\ref{sec:exp} reports the main experiments, and Section~\ref{sec:conc} concludes with limitations and future directions.

\section{Method: STRIDE}
\label{sec:method}
In this section, we present STRIDE, a two-stage framework that (i) encodes a short history of sparse point-sensor measurements into a latent state with a temporal encoder and (ii) decodes this latent state into a continuous spatial field via a modulated implicit neural representation (INR). We describe the observation model, the two-stage mapping, architectural instantiations, and the training procedure.

\subsection{Setting and Observation Model}
Consider a \emph{parametric} spatiotemporal dynamical system (e.g., a PDE) on a spatial domain $\Omega\subset\mathbb{R}^{d_\xi}$ indexed by a parameter $\bsmu\in\mathcal{P}\subset\mathbb{R}^{d_\mu}$. For a fixed $\bsmu$, the state at discrete time index $t$ is a vector-valued field $\bsx_t(\cdot;\bsmu):\Omega\to\mathbb{R}^{d_x}$; we write its pointwise value as $\bsx(\bsxi,t;\bsmu):=\bsx_t(\bsxi;\bsmu)$ for $\bsxi\in\Omega$.

We observe the system through ${N_s}$ point sensors at locations $\{\bsxi^{(1)},\ldots,\bsxi^{(N_s)}\}\subset\Omega$ and a fixed channel projection $\Pi:\mathbb{R}^{d_x}\to\mathbb{R}^{d_o}$ selecting the observed $d_o$ channels. The measurement vector at time $t$ is
\begin{equation}
    \bsy_t(\bsmu):=\bigl(\Pi\,\bsx(\bsxi^{(1)},t;\bsmu),\ldots,\Pi\,\bsx(\bsxi^{(N_s)},t;\bsmu)\bigr)\in\mathbb{R}^{p},\quad  p=N_s d_o.
\end{equation}
Given an observation window $\bsy_{t-k:t}(\bsmu)=\{\bsy_{t-k}(\bsmu),\dots,\bsy_t(\bsmu)\}$ and a query point $\bsxi$, our goal is to reconstruct the full field at time $t$, i.e., to predict $\bsx(\bsxi,t;\bsmu)$ for arbitrary $\bsxi$, uniformly over $\bsmu\in\mathcal{P}$. We often omit the dependence on $\bsmu$ for readability. At inference time, $\bsmu$ is not assumed to be observed, and STRIDE takes only the measurement window $\bsy_{t-k:t}=\{\bsy_{t-k},\dots,\bsy_t\}$ as input.

\subsection{Two-Stage Architecture}
STRIDE consists of a temporal (window) encoder $\mathcal{G}$ that maps the observation window $\bsy_{t-k:t}$ to a latent state $\bsz_t\in\mathbb{R}^{d_z}$ and a modulated INR decoder $\mathcal{F}$ that maps query locations $\bsxi$ and $\bsz_t$ to reconstructed field values.

In the first stage, the encoder $\mathcal{G}$ processes the sparse sensor measurements within a fixed window $\bsy_{t-k:t}$ and outputs $\bsz_t$. When $\mathcal{G}$ is instantiated as a recurrent model (e.g., LSTM/GRU), it evolves hidden states $\bsh_i\in\mathbb{R}^{d_z}$ autoregressively, and we take the final hidden state as the latent state. Concretely,
\begin{equation}
    \bsh_i = g(\bsh_{i-1}, \bsy_i), \quad i=t-k,\dots,t,\quad \bsh_{t-k-1} = \boldsymbol{0},
\end{equation}
where $g$ is the recurrent update. For early times, we pad missing past observations with zeros, i.e., $\bsy_i=\boldsymbol{0}$ for $i\le 0$. In general (including non-recurrent window encoders), we write $\bsz_t=\mathcal{G}(\bsy_{t-k:t})$.

In the second stage, the modulated INR decoder $\mathcal{F}$ takes a query location $\bsxi$ and the latent state $\bsz_t$ and outputs the corresponding field value. Specifically, we reconstruct $\bsx(\bsxi, t) \in \mathbb{R}^{d_x}$ as
\begin{equation}
    \hat{\bsx}(\bsxi, t) = \mathcal{F}(\bsxi, \bsz_t),
\end{equation}
where $\mathcal{F}$ is realized as a coordinate-based INR with shared parameters and lightweight \emph{shift modulations} (FiLM) predicted from $\bsz_t$ by a small hypernetwork. This design supports irregular meshes and arbitrary resolutions at inference by evaluating $\mathcal{F}$ at any set of query points.

\subsection{Architectural Choices}
\textbf{Temporal encoder $\mathcal{G}$.}
Motivated by SHRED \cite{williams_sensing_2024} and SHRED-ROM \cite{tomasetto_reduced_2025}, we use an LSTM as the default temporal model.
STRIDE is agnostic to this choice, and we also evaluate alternative recurrent encoders including GRU \cite{chung_empirical_2014} and Mamba \cite{gu_mamba_2024}.
To isolate the role of recurrence, we additionally consider non-recurrent window encoders such as an MLP applied to concatenated observations and a multi-head self-attention encoder over the observation sequence.

\textbf{INR decoder $\mathcal{F}$.}
We investigate two modulated INR backbones: SIREN \cite{sitzmann_implicit_2020} and FMMNN \cite{zhang_fourier_2025}.
The standard SIREN architecture is defined as
\begin{equation}
    f_\theta(\bsxi)=\bsW_L\left(\sigma_{L-1} \circ \sigma_{L-2} \circ \cdots \circ \sigma_0(\bsxi)\right)+\bsb_L, \text { with } \sigma_i\left(\bseta_i\right)=\sin \left(\omega_0\left(\bsW_i \bseta_i+\bsb_i\right)\right)
\end{equation}
where $\theta=\{\bsW_i,\bsb_i\}_{i=0}^L$ are trainable parameters and $\omega_0$ is a hyperparameter controlling the frequency bandwidth.
To condition the INR on a latent state $\bsz$, we employ FiLM-style shift modulations \cite{perez_film_2017}, which provide a lightweight conditioning mechanism. The resulting modulated SIREN is
\begin{equation}
    f_{\theta,\phi}(\bsxi)=\bsW_L\left(\sigma_{L-1} \circ \sigma_{L-2} \circ \cdots \circ \sigma_0(\bsxi)\right)+\bsb_L, \text { with } \sigma_i\left(\bseta_i\right)=\sin \left(\omega_0\left(\bsW_i \bseta_i+\bsb_i+\bsphi_i\right)\right)
\end{equation}
where $\theta$ are shared parameters and $\phi = \{\bsphi_i\}_{i=0}^{L-1}$ are shift modulations mapped from $\bsz$ via a (typically linear) hypernetwork. We follow standard operator-learning implementations of FiLM-style modulation (e.g., \citealp{serrano_operator_2023}).

The Fourier Multi-Component and Multi-Layer Neural Network (FMMNN) can serve as an INR backbone in the same manner as SIREN, while offering improved optimization behavior for high-frequency structures \cite{zhang_fourier_2025}.
\begin{equation}
    f_\theta(\bsxi)=\sigma_{L} \circ \sigma_{L-1} \circ \cdots \circ \sigma_1(\bsxi), \text { with } \sigma_i\left(\bseta_i\right)=\bsA_i \sin \left(\bsW_i \bseta_i+\bsb_i\right) + \bsc_i
\end{equation}
where $\{\bsW_i,\bsb_i\}$ are randomly initialized and fixed, while $\{\bsA_i,\bsc_i\}$ are trainable. In each layer, the number of components $d_i$ is known as the rank, and the number of hidden neurons $n_i$ is referred to as the layer width; typically $n_i\gg d_i$ to provide a rich set of random basis functions.
To condition the network on some latent state $\bsz$, similar shift modulations are applied to FMMNN.
\begin{equation}
    f_{\theta,\phi}(\bsxi) = \sigma_{L} \circ \sigma_{L-1} \circ \cdots \circ \sigma_1(\bsxi), \text { with } \sigma_i\left(\bseta_i\right)=\bsA_i \sin \left(\bsW_i \bseta_i+\bsb_i\right) + \bsc_i + \bsphi_i
\end{equation}
where $\phi = \{\bsphi_i\}_{i=1}^{L-1}$ are shift modulations applied to each rank, mapped from $\bsz$ via a trainable linear hypernetwork.

To enhance the reconstruction, Fourier encoding \cite{qiu_derivative-enhanced_2024} is incorporated to better capture high-frequency spatial components. Specifically, a trainable parameter matrix $P$ is included to map the spatial coordinates to a higher-dimensional embedding $\bsxi \mapsto [\bsxi, \cos(P\bsxi), \sin(P\bsxi)]$, which serves as the input to the decoder.

\begin{figure*}[htb]
  \begin{center}
    \centerline{\includegraphics[width=\textwidth]{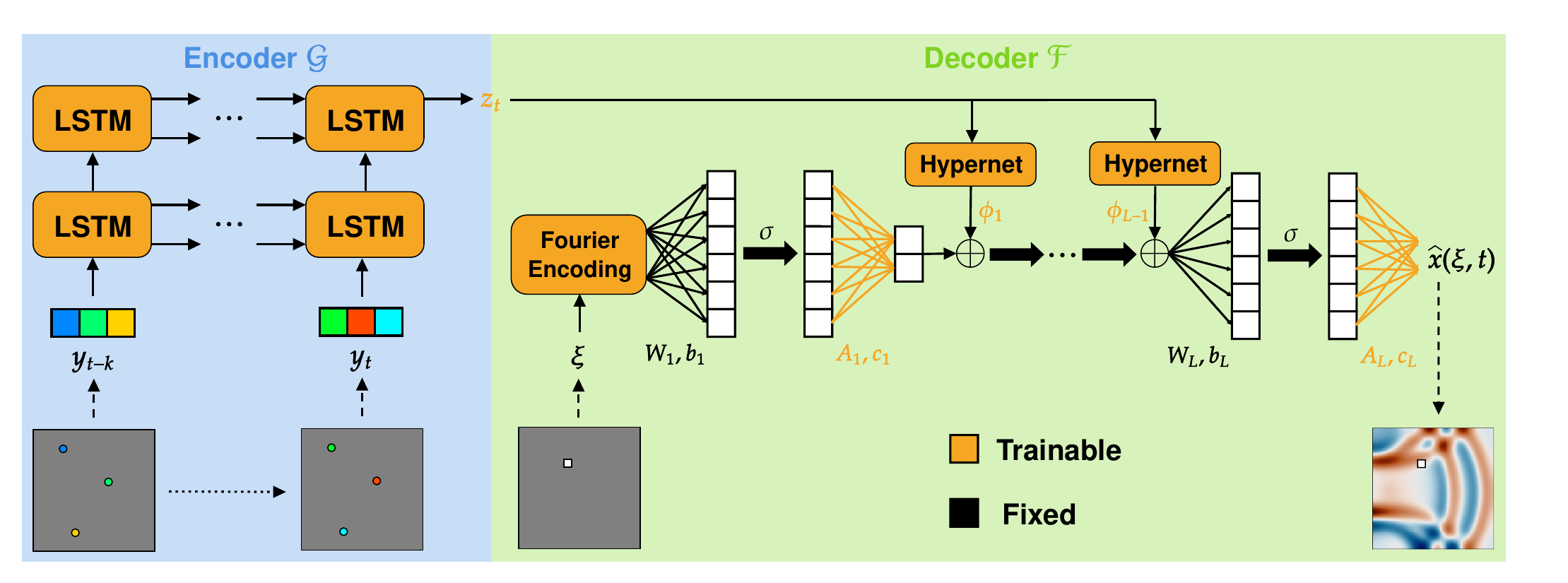}}
    \caption{Overview of STRIDE. A temporal encoder maps a window of point-sensor observations $\bsy_{t-k:t}$ to a latent state $\bsz_t$. A conditional spatial decoder, e.g., modulated INR decoder (FMMNN with shift modulation), then evaluates $\hat{\bsx}(\bsxi,t)$ at arbitrary query locations $\bsxi$ (optionally Fourier-encoded), enabling discretization- and resolution-invariant reconstruction across parameterized trajectories.}
    \label{fig:diagram}
  \end{center}
  \vskip -0.2in
\end{figure*}

\subsection{Training and Normalization}
We adopt an end-to-end training strategy to jointly optimize the parameters of the temporal encoder $\mathcal{G}$ and the INR decoder $\mathcal{F}$. The training data consist of input-output pairs of the form $\{\bsy_{t-k:t}, \bsxi, \bsx(\bsxi, t)\}$, where $\bsxi$ denotes query locations and $\bsx(\bsxi,t)$ the corresponding ground-truth field values. We minimize the following Mean Squared Error (MSE):
\begin{equation}
    \mathcal{L} = \frac{1}{N_{tr} N_t N_\xi}\sum_{j=1}^{N_{tr}}\sum_{t=1}^{N_t}\sum_{q=1}^{N_\xi}     \left\|\bsx\!\left(\bsxi_{j,t}^{(q)}, t;\bsmu_j\right) -     \hat{\bsx}\!\left(\bsxi_{j,t}^{(q)}, t, \bsmu_j\right)\right\|^2,
\end{equation}
where $j$ indexes training trajectories generated at parameters $\{\bsmu_j\}_{j=1}^{N_{tr}}$ and $N_t$ is the number of time steps per trajectory. For each snapshot $(j,t)$, we sample $N_\xi$ query points $\{\bsxi_{j,t}^{(q)}\}_{q=1}^{N_\xi}$ from the available spatial discretization (or directly from $\Omega$) and compute predictions via $\hat{\bsx}(\bsxi,t,\bsmu_j)=\mathcal{F}\!\left(\bsxi,\mathcal{G}(\bsy_{t-k:t}(\bsmu_j))\right)$. Unlike SHRED, which matches full spatial grids at each step, randomized spatial sampling reduces training cost and directly enables discretization- and resolution-invariant learning. To facilitate fast and stable training, we use the SOAP optimizer \cite{vyas_soap_2025} with weight decay, along with the ReduceLROnPlateau scheduler.

Prior to training, we perform channel-wise normalization similar to \cite{regazzoni_learning_2024}. Specifically, we scale the state variables $\bsx$ (and thus the sensor measurements $\bsy$) and the spatial coordinates $\bsxi$ such that each component resides within the range $[-1,1]$. This normalization helps improve numerical stability, especially for early-time snapshots and boundary regions that may contain many near-zero values.

\section{Theoretical Justification}
\label{sec:theory}
We provide a conditional justification for why reconstructing the full field from a \emph{finite} window of point sensors is plausible, and why a two-stage architecture (temporal encoder + spatial decoder) is expressive enough to approximate the associated reconstruction operator.

\subsection{Setting}
Let $\Omega\subset\mathbb{R}^{d_\xi}$ be a compact spatial domain and let ${d_x}$ be the number of field channels. Consistent with Section~\ref{sec:method}, we denote the full state at discrete time index $t$ by the vector-valued field $\bsx_t(\cdot)\in X$, and write its pointwise value as $\bsx(\bsxi,t):=\bsx_t(\bsxi)\in\mathbb{R}^{d_x}$. We make the dependence on the parameter explicit by $\bsx_t(\cdot;\bsmu)$ and $\bsx(\bsxi,t;\bsmu)$.
Here $X$ is a Banach space continuously embedded in $C^0(\Omega,\mathbb{R}^{d_x})$ so that point evaluation is continuous, e.g., $X=H^{s}(\Omega,\mathbb{R}^{d_x})$ with $s>d_\xi/2$ by Sobolev embedding.

In our numerical benchmarks, trajectories arise from a \emph{parametric} PDE: there is a compact parameter set $\mathcal{P}\subset\mathbb{R}^{d_\mu}$ and each trajectory is generated under some fixed $\bsmu\in\mathcal{P}$ (e.g., geometry/forcing/initial-condition parameters). Let $\{S_\bsmu^t\}_{t\ge 0}$ be the (semi)flow on $X$ induced by the PDE at parameter $\bsmu$, and fix a sampling interval $\Delta t>0$ with discrete-time map
$
    F_\bsmu := S_\bsmu^{\Delta t}.
$
Thus, for a fixed $\bsmu$, $\bsx_{t+1}=F_\bsmu\bsx_t$ for integer $t\ge 0$.
We assume that for each $\bsmu\in\mathcal{P}$ the dynamics admit a compact forward-invariant set $\mathcal{A}_\bsmu\subset X$ (e.g., a parameter-dependent global attractor in dissipative settings) on which the solution map is regular enough for delay embedding arguments.
To unify multiple-trajectory learning across parameters, we consider the augmented invariant set in $\mathcal{P}\times X$,
\begin{equation}
    \mathcal{A}:=\bigl\{(\bsmu,\bsx):\ \bsmu\in\mathcal{P},\ \bsx\in\mathcal{A}_\bsmu\bigr\}\subset \mathcal{P}\times X,
\end{equation}
equipped with the product norm $\|(\bsmu,\bsx)-(\bsmu',\bsx')\|_{\mathcal{P}\times X}:=\|\bsmu-\bsmu'\|+\|\bsx-\bsx'\|_X$.

\subsection{Point-sensor observation and delay observability}
Fix ${N_s}$ sensor locations $\{\bsxi^{(1)}, \dots, \bsxi^{(N_s)}\}\subset\Omega$ and let $d_o$ denote the number of observed channels at each sensor, so the observation dimension is $p=N_s\times d_o$. Let $\Pi:\mathbb{R}^{d_x}\to\mathbb{R}^{d_o}$ be the fixed linear projection selecting the observed channels. Define the point-sensor map $h:X\to\mathbb{R}^{p}$ by stacking projected sensor values as
\begin{equation}
    h(\bsx)=\bigl(\Pi\bsx(\bsxi^{(1)}),\ldots,\Pi\bsx(\bsxi^{(N_s)})\bigr)\in\mathbb{R}^{p}.
\end{equation}
For an integer lag $k\ge 0$, define the \emph{forward} delay map $\Phi_k:\mathcal{A}\to\mathbb{R}^{(k+1)p}$ on the augmented state by
\begin{equation}
    \Phi_k(\bsmu,\bsx)=\bigl(h(\bsx),\,h(F_\bsmu\bsx),\,\ldots,\,h(F_\bsmu^{k}\bsx)\bigr).
\end{equation}
In data terms, for a trajectory generated under a fixed $\bsmu$, $\bsx_t=F_\bsmu^t \bsx_0$ (with integer time index), define the per-step sensor vector $\bsy_t:=h(\bsx_t)\in\mathbb{R}^{p}$. Then the observation window $\bsy_{t-k:t}$ equals $\Phi_k(\bsmu,\bsx_{t-k})$.

\begin{assumption}[Finite-dimensional long-term complexity]\label{assump:finite-dim}
The augmented invariant set $\mathcal{A}\subset\mathcal{P}\times X$ is compact and has finite box-counting dimension $d_A:=\dim_B(\mathcal{A})<\infty$.
\end{assumption}

Assumption~\ref{assump:finite-dim} is satisfied by broad classes of \emph{dissipative} PDEs for which the solution (semi)flow possesses a compact global attractor with finite fractal/box-counting dimension; classical examples include the 1D Kuramoto--Sivashinsky equation, the 2D Navier--Stokes equations, and geophysical flows (see, e.g., \citealp{temam1997infinite_dimensional}).
In our benchmarks, the KS, FlowAO, and SWE settings are dissipative, and the seismic-wave setup includes an absorbing damping mask that introduces dissipation in the simulated dynamics.
Moreover, our datasets are parametric: varying $\bsmu$ (e.g., forcing/geometry profiles, initial conditions, source location, velocity map) traces out a \emph{family} of invariant sets $\{\mathcal{A}_\bsmu\}_{\bsmu\in\mathcal{P}}$; Assumption~\ref{assump:finite-dim} asserts that the resulting augmented set $\mathcal{A}\subset\mathcal{P}\times X$ remains effectively finite-dimensional.

\begin{assumption}[Stable delay observability]\label{assump:delay-obs}
There exist a lag $k$ and constants $L_\Phi,L_\Psi>0$ such that $\Phi_k$ is injective on $\mathcal{A}$ and has an inverse $\Psi:Y\to\mathcal{A}$ on $Y:=\Phi_k(\mathcal{A})$ satisfying the bi-Lipschitz-type stability bounds
\begin{equation}
    \|\Phi_k(\bsa)-\Phi_k(\bsa')\| \le L_\Phi \|\bsa-\bsa'\|_{\mathcal{P}\times X},
\end{equation}
with $\bsa = (\bsmu,\bsx)$ and $\bsa' = (\bsmu',\bsx')$ and 
\begin{equation}
    \|\Psi(\bsy)-\Psi(\bsy')\|_{\mathcal{P}\times X} \le L_\Psi \|\bsy-\bsy'\|.
\end{equation}
Equivalently, the inverse is continuous and stable to perturbations of the observation window.
\end{assumption}

Assumption~\ref{assump:delay-obs} is a \emph{stable} observability/embedding condition for the delay-coordinate map on the augmented set $\mathcal{A}$. 
For infinite-dimensional systems such as dissipative PDEs,  
the Ma\~n\'e projection theorem shows that compact sets of finite box-counting dimension in Banach spaces admit injective linear projections into $\mathbb{R}^{d_z}$ (with Hölder continuous inverse) when $d_z$ is sufficiently large \cite{mane1981dimension}.
In our setting, $h$ is a fixed \emph{point-sensor} map and does not directly observe $\bsmu$, so genericity is not automatic; we therefore treat Assumption~\ref{assump:delay-obs} as an identifiability assumption on $(\bsmu,\bsx)$ and empirically examine parameter identification.

\subsection{A continuous reconstruction operator}
Define the (ground-truth) field at time index $t$ as $\bsx(\cdot,t):=\bsx_t(\cdot)\in C^0(\Omega,\mathbb{R}^{d_x})$.
Under Assumption~\ref{assump:delay-obs}, an observation window $\bsy\in Y$ identifies the corresponding augmented state $(\bsmu,\bsx)=\Psi(\bsy)\in\mathcal{A}$ at the \emph{start} of the window. The field at the end of the window is then $F_\bsmu^k \bsx$. This defines the reconstruction operator $T:Y\to C^0(\Omega,\mathbb{R}^{d_x})$:
\begin{equation}
    T(\bsy)(\bsxi) := \bigl(F_\bsmu^{k}\bsx\bigr)(\bsxi),\quad \bsy\in Y,\, \bsxi\in\Omega.
\end{equation}

\begin{lemma}[Continuity and stability of $T$]\label{lem:stability-T}
If the map $(\bsmu,\bsx)\mapsto F_\bsmu^k\bsx$ is continuous on $\mathcal{A}$ and point evaluation $\bsx\mapsto \bsx(\bsxi)$ is continuous on $X$, then $T$ is continuous on the compact set $Y$. Moreover, if $(\bsmu,\bsx)\mapsto F_\bsmu^k\bsx$ is $L_{F^k}$-Lipschitz on $\mathcal{A}$ (with respect to $\|\cdot\|_{\mathcal{P}\times X}$) and the embedding $X\hookrightarrow C^0(\Omega,\mathbb{R}^{d_x})$ satisfies
$\|\bsx\|_{C^0}:=\sup_{\bsxi\in\Omega}\|\bsx(\bsxi)\|\le C_{\mathrm{ev}}\|\bsx\|_X,\; \forall \bsx\in X,$
then
\begin{equation}
    \sup_{\bsxi\in\Omega}\|T(\bsy)(\bsxi)-T(\bsy')(\bsxi)\|
    \le C_{\mathrm{ev}}\,L_{F^k}\,L_\Psi\,\|\bsy-\bsy'\|.
\end{equation}
\end{lemma}
We provide the proof in Appendix~\ref{sec:proof-lemma}.  Stable delay observability implies that small perturbations of the sensor window (noise, discretization) induce controlled reconstruction errors.

\subsection{Approximation by a STRIDE-type architecture}
We now connect the operator $T$ to STRIDE's structure. Let $\mathcal{G}:Y\to Z$ be a temporal encoder that maps observation windows to a latent state $z \in Z$, and let $\mathcal{F}:\Omega\times Z\to\mathbb{R}^{d_x}$ be a conditional spatial decoder. STRIDE represents reconstructions as $\hat{\bsx}(\bsxi)=\mathcal{F}(\bsxi,\mathcal{G}(\bsy))$.

\begin{theorem}[Operator factorization and approximation]\label{thm:operator-approx}
Assume Assumptions~\ref{assump:finite-dim}--\ref{assump:delay-obs} and that $G_k(\bsmu,\bsx):=F_\bsmu^k\bsx$ is continuous on $\mathcal{A}$. Let $\mathcal{A}_X$ be the projection of $\mathcal{A}$ onto $X$ and set $d_X:=\dim_B(\mathcal{A}_X)$.
Then there exists a compact set $Z\subset\mathbb{R}^{d_z}$ with ${d_z}>2d_X$ and continuous maps 
$\mathcal{G}:Y\to Z$ and $\mathcal{F}:\Omega\times Z \to\mathbb{R}^{d_x}$ of STRIDE form such that
\begin{equation}
    \sup_{\bsy\in Y}\sup_{\bsxi\in\Omega}\bigl\|T(\bsy)(\bsxi)-\mathcal{F}(\bsxi,\mathcal{G}(\bsy))\bigr\| < \varepsilon,
\end{equation}
where $\mathcal{G}$ and $\mathcal{F}$ can be uniformly approximated on compact sets by standard feedforward neural networks. 
\end{theorem}
The proof is provided in Appendix~\ref{sec:proof-theorem}.
Though the formal approximation guaranty is for standard feedforward networks, in practice, we instantiate $\mathcal{G}$ with a windowed recurrent encoder and $\mathcal{F}$ with a modulated INR; these architectures are more expressive in our experiments.

\section{Experiments}
\label{sec:exp}
We evaluate STRIDE on four benchmarks that reconstruct spatiotemporal fields from sparse point-sensor observations. The benchmarks range from chaotic dynamics to wave propagation and all involve multiple trajectories driven by varying initial conditions and/or physical parameters. In this section, we denote STRIDE-FMMNN as our default instantiation with an FMMNN INR backbone and STRIDE-SIREN as the SIREN-backbone variant.

\subsection{Experimental Setup}

\paragraph{Kuramoto--Sivashinsky Equation.} 
The Kuramoto--Sivashinsky (KS) equation is a nonlinear PDE that exhibits spatiotemporal chaos and strong sensitivity to initial conditions, as illustrated in Fig.~\ref{fig:kstraj}, where the system exhibits divergent dynamics despite starting from nearly identical initial conditions. Following \cite{tomasetto_reduced_2025}, we generate data using the Exponential Time Differencing fourth-order Runge--Kutta (ETDRK4) method with periodic boundary conditions on $N_\xi=100$ spatial points for $N_t=201$ snapshots. We fix the viscosity to 1.0 and vary only the initial condition, generating 1,000 trajectories.
\begin{figure*}[htb]
  \begin{center}
    \centerline{\includegraphics[width=0.33\textwidth]{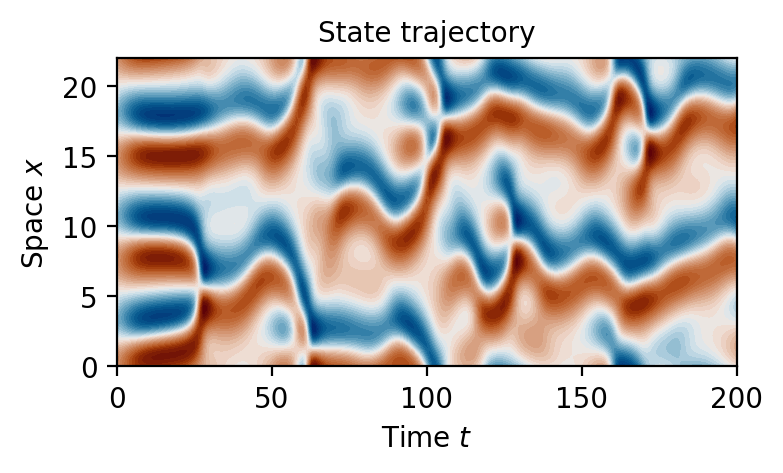}
    \includegraphics[width=0.33\textwidth]{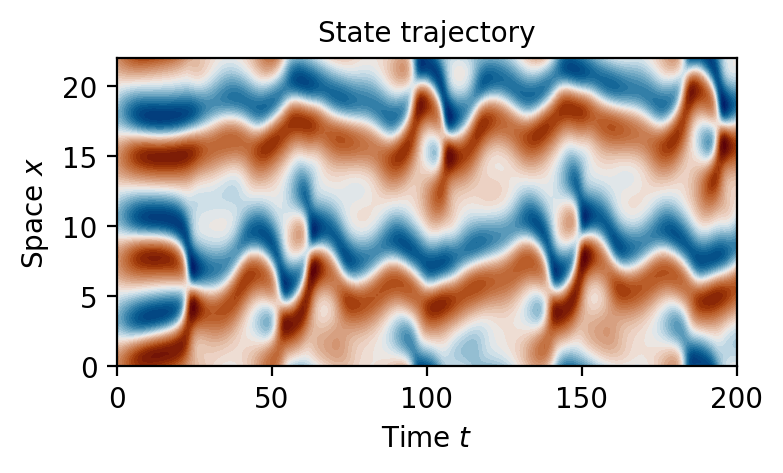}
    \includegraphics[width=0.33\textwidth]{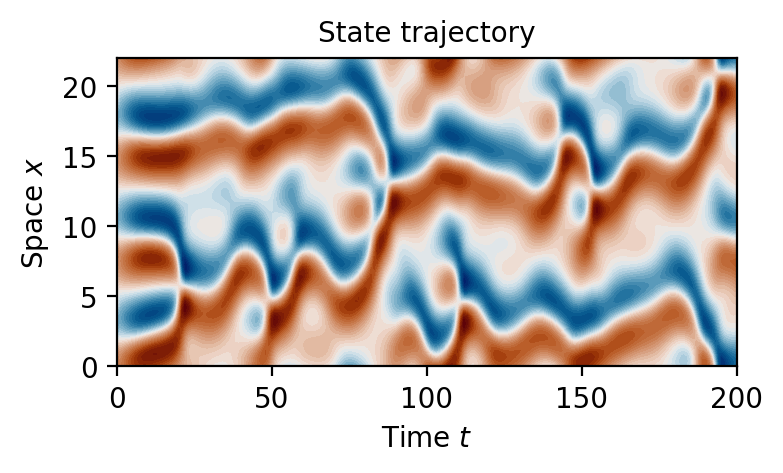}}
    \caption{Divergent trajectories in the chaotic KS system. The spatiotemporal fields are initialized with nearly identical frequency parameters ($2.99, 3.00, \text{ and } 3.01$), illustrating the system's sensitive dependence on initial conditions.}
    \label{fig:kstraj}
  \end{center}
    \vskip -0.2in
\end{figure*}

\paragraph{Flow Around an Obstacle.}
Flow Around an Obstacle (FlowAO) is an incompressible flow benchmark with time-dependent physical and geometric parameters. Following \cite{tomasetto_reduced_2025}, we generate 200 trajectories with $N_t=200$ time steps each by solving the unsteady Navier--Stokes equations. Each trajectory corresponds to a different time profile of the angle of attack and inflow intensity, as well as a different obstacle shape; see one example trajectory in Fig.~\ref{fig:flowaotraj}. Each snapshot contains $N_\xi=40,296$ spatial points per velocity component.
\begin{figure*}[htb]
  \begin{center}
    \includegraphics[width=0.49\textwidth]{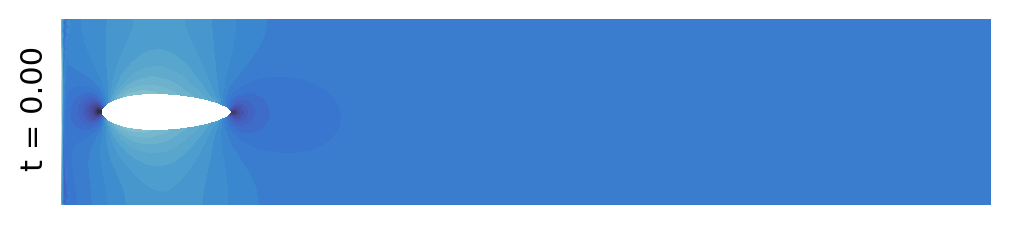}
    \includegraphics[width=0.49\textwidth]{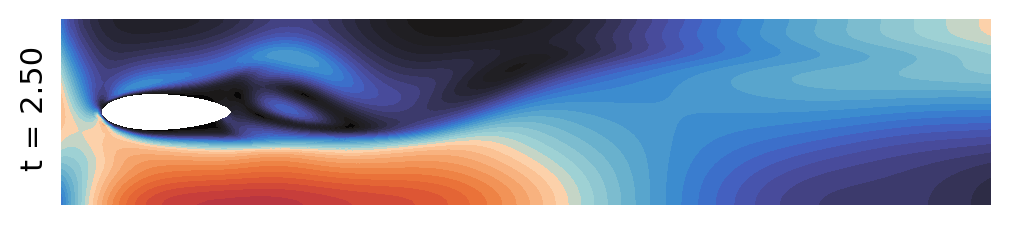}
    \includegraphics[width=0.49\textwidth]{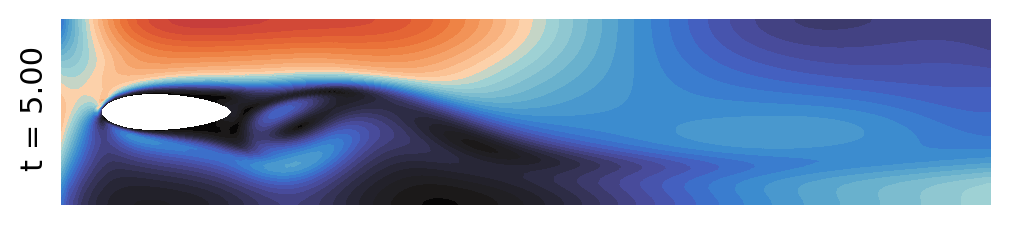}
    \includegraphics[width=0.49\textwidth]{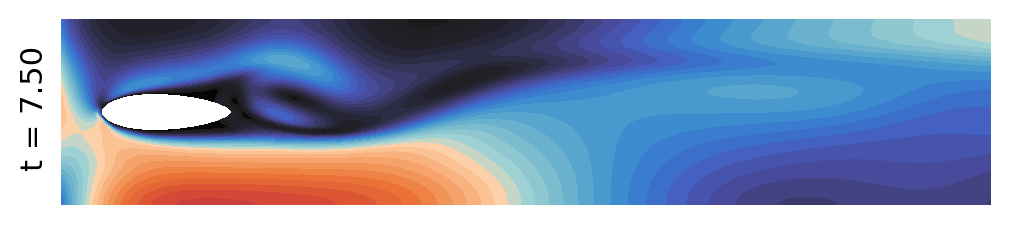}
    \includegraphics[width=0.49\textwidth]{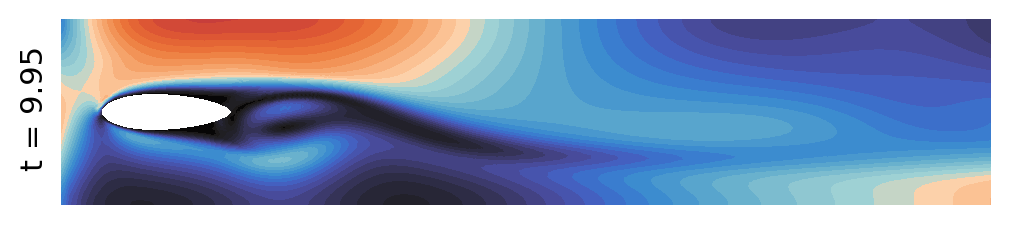}
    \caption{Example trajectory of the velocity field in the FlowAO case, showing unsteady flow patterns influenced by time-varying angles of attack and inflow intensities.}
    \label{fig:flowaotraj}
  \end{center}
    \vskip -0.1in
\end{figure*}

\paragraph{Shallow Water Equations.}
The shallow water equations (SWE) model free-surface flow, e.g., as used in tsunami modeling shown in Fig.~\ref{fig:swetraj}. We generate 200 trajectories with $N_t=201$ time steps, initialized with Gaussian bumps of surface elevation at random locations, following \cite{si_latent-ensf_2024}. The spatial domain is discretized uniformly at two resolutions: $64 \times 64$ and $128 \times 128$.

\begin{figure*}[htb]
  \begin{center}
    \centerline{\includegraphics[width=\textwidth]{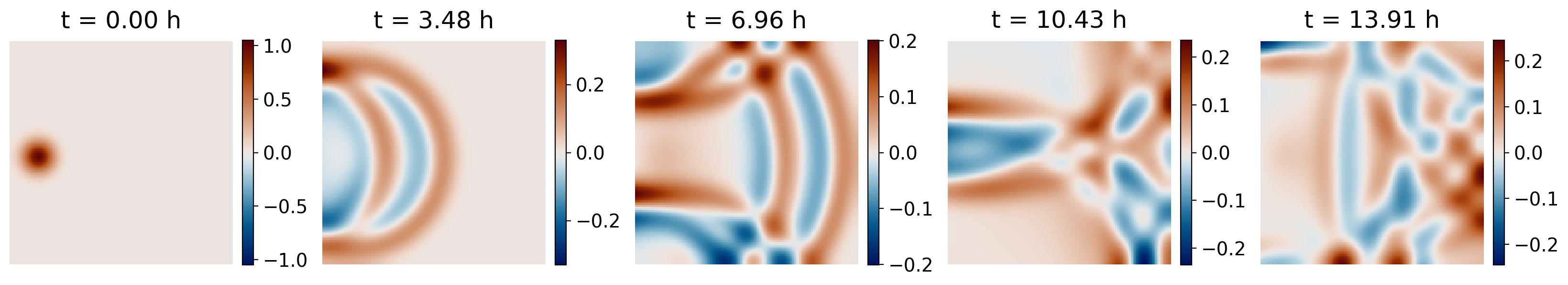}}
    \caption{Example trajectory of the surface elevation in the SWE case, demonstrating the propagation of a tsunami-like wave generated from a Gaussian source.}
    \label{fig:swetraj}
  \end{center}
    \vskip -0.2in
\end{figure*}

\paragraph{Seismic Wave Propagation.}
\begin{wrapfigure}{r}{0.35\textwidth}
  \begin{center}
    \vspace{-5pt} 
    \includegraphics[width=0.27\textwidth]{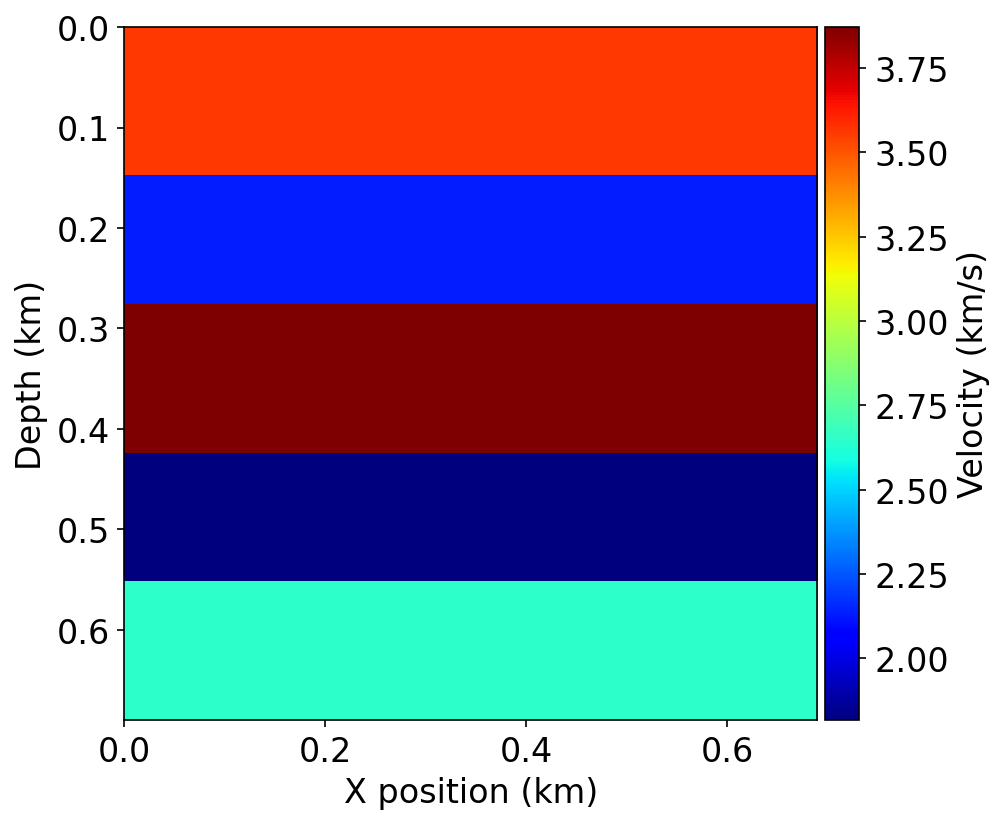}
    \caption{Sample of velocity map.}
    \label{fig:seismic}
  \end{center}
\end{wrapfigure}
We consider wavefields governed by the acoustic wave equation. Using Devito \cite{devito-compiler}, we generate 1,000 trajectories with $N_t=201$ time steps, where a 10\,Hz Ricker source is placed at random locations in a multi-layer velocity map sourced from the “FlatVel-B” (FVB) dataset from OpenFWI \cite{deng_openfwi_2023}, see one sample in Fig.~\ref{fig:seismic}. The spatial resolution is $70 \times 70$, with an absorbing damping mask applied at the domain boundaries to mimic an infinite domain. Layer boundaries induce complex wave phenomena, including refraction and reflection, thereby significantly complicating the reconstruction task. One example of the wave propagation is shown in Fig.~\ref{fig:seismictraj}.

\begin{figure*}[htb]
  \begin{center}
    \centerline{\includegraphics[width=\textwidth]{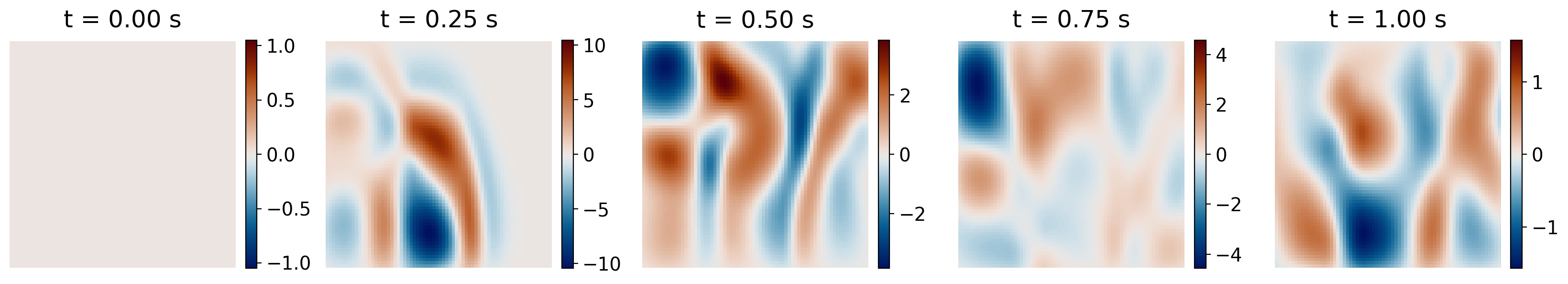}}
    \caption{Example trajectory of acoustic wave propagation across a multi-layer velocity map, illustrating complex reflections and refractions.}
    \label{fig:seismictraj}
  \end{center}
    \vskip -0.2in
\end{figure*}

\subsection{Evaluation Metrics}
The generated trajectories are partitioned into 80\%/10\%/10\% splits for training, validation, and testing. During training and validation, we may subsample spatial query points as described in Section~\ref{sec:method}; at test time, we evaluate reconstructions on all available spatial points. For the KS and FlowAO cases, we report the mean relative $\ell_2$ error for each state variable,
\begin{equation}
    \epsilon = \frac{1}{N_{\text{test}} N_t}\sum_{j=1}^{N_{\text{test}}}\sum_{t=1}^{N_t} \frac{\|\bsx_{j,t}-\hat{\bsx}_{j,t}\|_2}{\|\bsx_{j,t}\|_2},
\end{equation}
where $\|\cdot\|_2$ is the Euclidean norm over the spatial discretization and $N_{\text{test}}$ is the number of test trajectories. Due to the wide dynamic range of the wavefields, we report the following overall relative Frobenius error for the SWE and Seismic cases,
\begin{equation}
    \epsilon = \frac{\|\bsx-\hat{\bsx}\|_F}{\|\bsx\|_F},
\end{equation}
where $\bsx$ and $\hat{\bsx}$ are matrices formed by stacking the true and predicted spatial fields (over time and space) of the test samples.

\subsection{Numerical Results}

Table~\ref{tab:modelcomp} summarizes reconstruction errors and parameter counts for our method and baselines across the four datasets. STRIDE-SIREN uses the same STRIDE architecture but replaces the INR backbone with a modulated SIREN. For FlowAO, we observe only the horizontal velocity component at sensor locations and reconstruct both velocity components; we report the error of the horizontal component. For SWE, we observe only the surface elevation and reconstruct the elevation and both velocity components; we report the error of the surface elevation. For STRIDE-FMMNN, we report parameters as (trainable/all) due to fixed random weights in FMMNN. Model hyperparameters are selected using Bayesian optimization on the validation set. To ensure a fair comparison of the decoder, all models use the same sensor measurements, time lag $k$, and latent dimension ${d_z}$. Please refer to Appendix~\ref{hyper} for details on hyperparameter search and training.

STRIDE-FMMNN achieves the lowest errors across all datasets. SHRED performs well on KS but degrades as the spatial dimensionality increases, since the parameter count is dominated by the last layer of its shallow decoder and scales with the output dimension. SHRED-ROM reduces parameters via Proper Orthogonal Decomposition (POD) but incurs larger errors in intrinsically high-dimensional systems. Across the four datasets, POD achieves mean relative reconstruction errors of 0.57\%, 0.45\%, 0.80\%, and 1.47\% on the test sets, employing 20, 150, 400, and 2000 modes, respectively. STRIDE-SIREN is a competitive baseline, but is consistently outperformed by STRIDE-FMMNN at comparable model sizes, especially in wave propagation problems.

\begin{table*}[htb]
  \caption{Model comparison of parameter count and reconstruction error ($\epsilon$) per dataset.}
  \label{tab:modelcomp}
  \begin{center}
    \begin{small}
      \begin{sc}
        \begin{tabular}{c c c c c c} 
            \toprule
            Model & dataset $\rightarrow$ & \emph{KS} & \emph{FlowAO} & \emph{SWE (128$\times$128)} & \emph{Seismic} \\
            \midrule
            \multirow{2}{*}{SHRED} & \#parameters & 425K & 32.5M & 20.1M & 3.09M \\
                         & Error $\epsilon$       & \underline{6.20e-2} & 4.13e-2 & 1.75e-1 & 1.67e-1 \\
            \midrule
            \multirow{2}{*}{SHRED-ROM} & \#parameters & 393K & 395K & 917K & 1.93M \\
                         & Error $\epsilon$        & 6.28e-2 & 6.53e-2 & 5.14e-1 & 4.87e-1 \\
            \midrule
            \multirow{2}{*}{\shortstack{STRIDE-\\SIREN}} & \#parameters & 448K & 121K & 397K & 1.49M \\
                         & Error $\epsilon$        & 6.34e-2 & \underline{3.12e-2} & \underline{4.62e-2} & \underline{1.12e-1} \\
            \midrule
            \multirow{2}{*}{\shortstack{STRIDE-\\FMMNN}} & \#parameters & 348K/484K & 103K/149K & 337K/407K & 1.39M/1.80M \\
                         & Error $\epsilon$        & \textbf{3.09e-2} & \textbf{2.97e-2} & \textbf{2.78e-2} & \textbf{8.41e-2} \\
            \bottomrule
        \end{tabular}
      \end{sc}
    \end{small}
  \end{center}
\end{table*}

The comparison among different models holds for a more challenging SWE variant. As shown in Table~\ref{tab:newSWE}, STRIDE-FMMNN remains the best-performing model when we double the simulated physical time (while keeping $N_t=201$) so that wave reflections introduce more complexity near the end of the trajectories.

\begin{table}[htb]
  \caption{Model comparison of parameters and error on the extended SWE (128$\times$128) dataset.}
  \label{tab:newSWE}
  \begin{center}
    \begin{small}
      \begin{sc}
        \begin{tabular}{l cccc} 
            \toprule
            Model & SHRED & SHRED-ROM & STRIDE-SIREN & STRIDE-FMMNN \\
            \midrule
            \#parameters     & 20.1M   & 917K    & 1.45M   & 918K/1.53M \\
            Error $\epsilon$ & 20.25\% & 57.09\% & \underline{2.89\%}  & \textbf{2.51\%} \\
            \bottomrule
        \end{tabular}
      \end{sc}
    \end{small}
  \end{center}
\end{table}

Fig.~\ref{fig:gtpred} compares STRIDE-FMMNN and the baselines on a test snapshot from extended (double in time) SWE (128$\times$128). As waves bounce back and forth in the physical domain, the surface elevation field $\eta$ becomes highly structured near the final time step. Given the same sparse observations, STRIDE-FMMNN substantially improves reconstruction accuracy over SHRED and STRIDE-SIREN.

\begin{figure*}[htb]
  \begin{center}
    \centerline{\includegraphics[width=\textwidth]{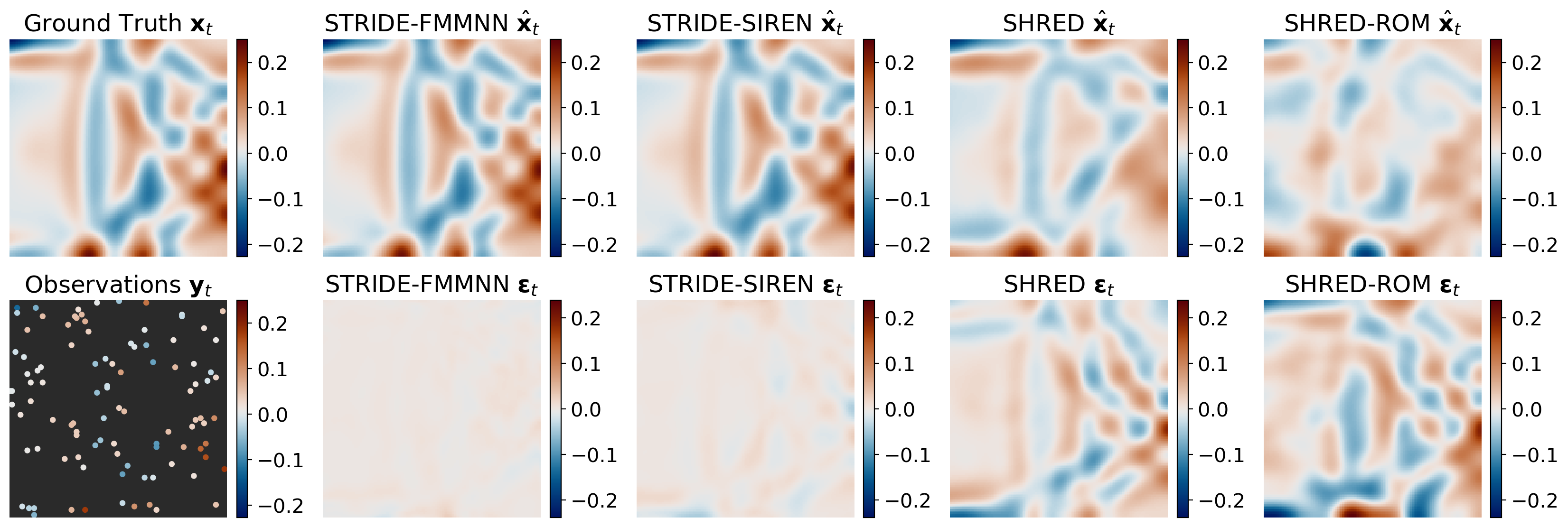}}
    \caption{Visualization of the surface elevation $\eta$ in the extended SWE (128$\times$128) case. The first column shows the ground truth at the last (200th) time step in an example test trajectory, along with observations at 100 random locations used as model input. The next columns show the prediction and error of STRIDE-FMMNN and baseline models for the same snapshot.}
    \label{fig:gtpred}
  \end{center}
    \vskip -0.2in
\end{figure*}

Fig.~\ref{fig:Seismicgtpred} provides a visual comparison between STRIDE-FMMNN and baseline models for a representative test snapshot from the Seismic dataset. The superiority of the proposed framework is particularly evident at the final time step, where the wavefield develops highly complex interference patterns resulting from cumulative refraction and reflection across multi-layered media.

\begin{figure*}[htb]
  \begin{center}
    \centerline{\includegraphics[width=\textwidth]{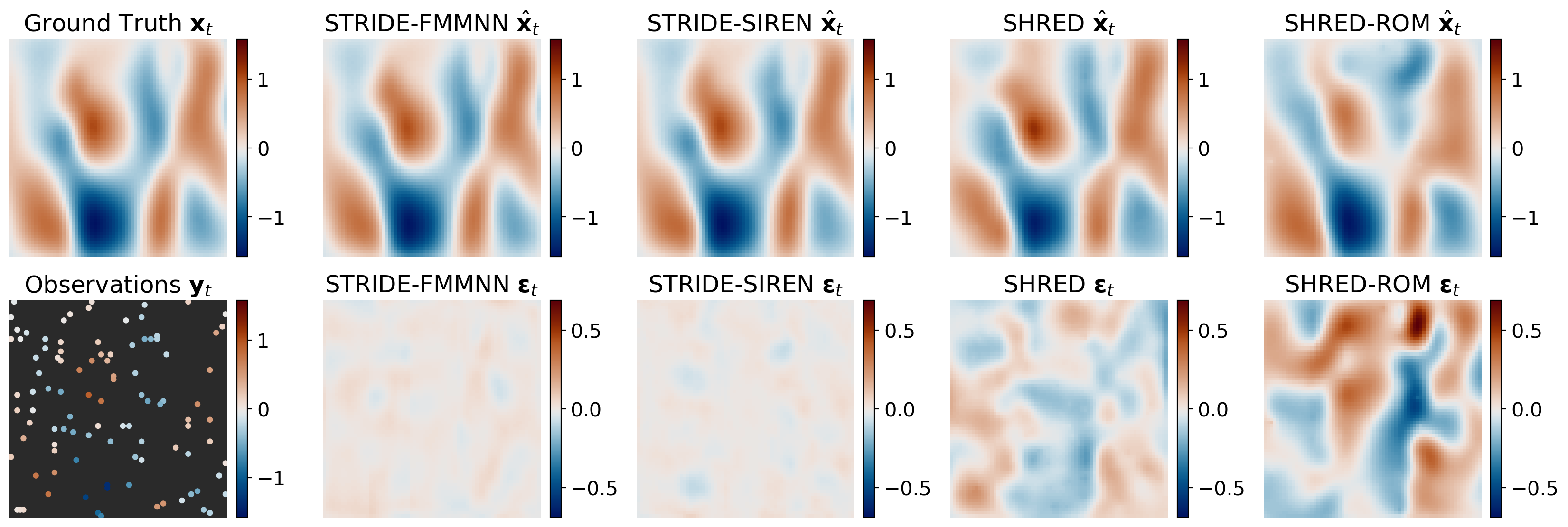}}
    \caption{Visual comparison of acoustic wave reconstruction under complex interference. (Left) Ground truth at $t=200$ with 100 random sparse observations. (Right) Predicted wavefields and corresponding error maps for STRIDE-FMMNN and baselines, highlighting our model's ability to capture intricate reflection and refraction patterns.}
    \label{fig:Seismicgtpred}
  \end{center}
    \vskip -0.2in
\end{figure*}

As shown in Fig.~\ref{fig:errplot}, STRIDE-FMMNN outperforms SHRED, STRIDE-SIREN, and SHRED-ROM at each time step for the KS and SWE cases. While all models exhibit larger errors at early times (limited history), STRIDE-FMMNN maintains a lower error as predictions become more challenging in chaotic dynamics (KS) and as spatial fields become increasingly complex (SWE).

\begin{figure}[htb]
  \begin{center}
    \centerline{\includegraphics[width=0.45\columnwidth]{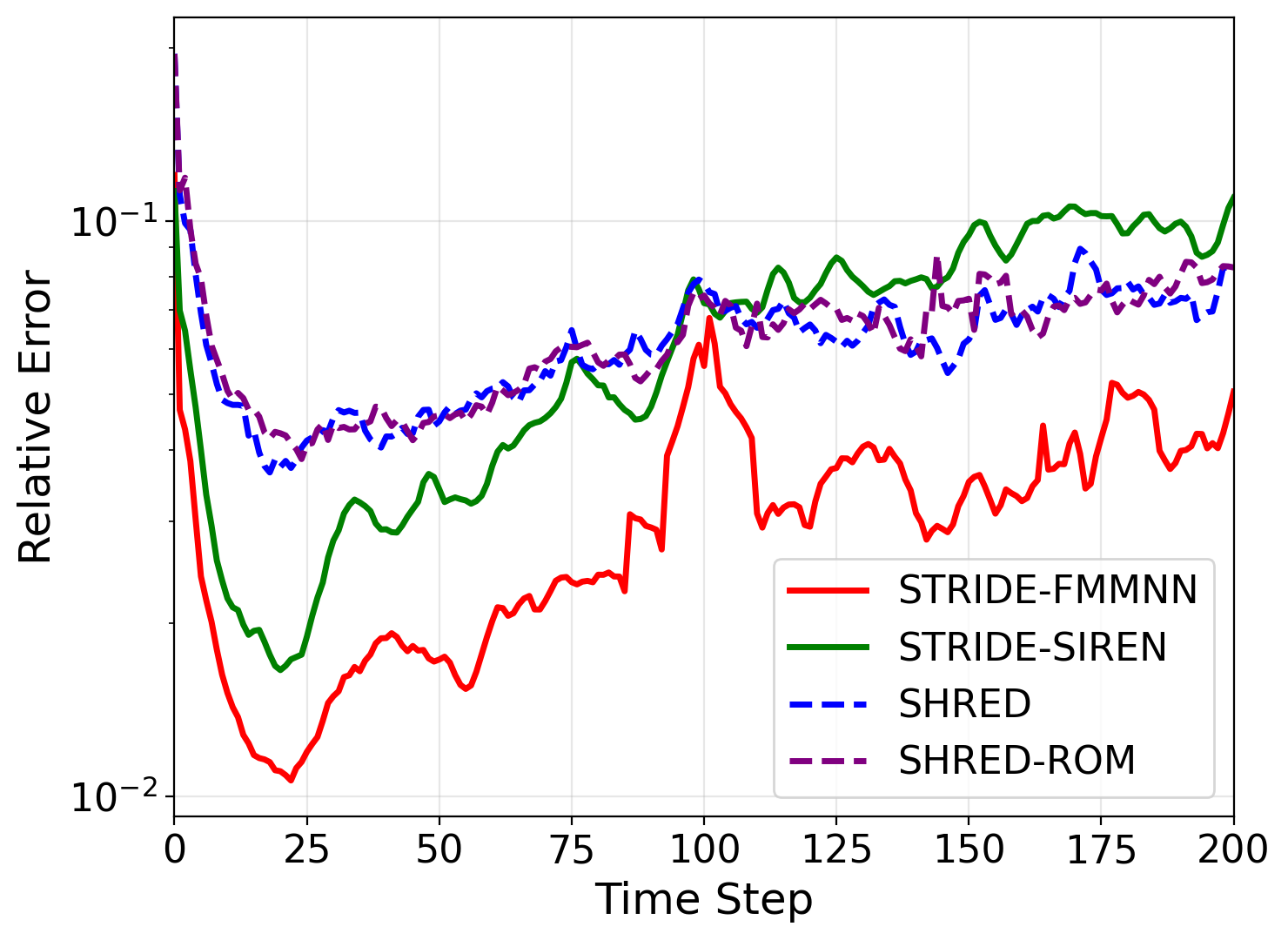}
    \includegraphics[width=0.45\columnwidth]{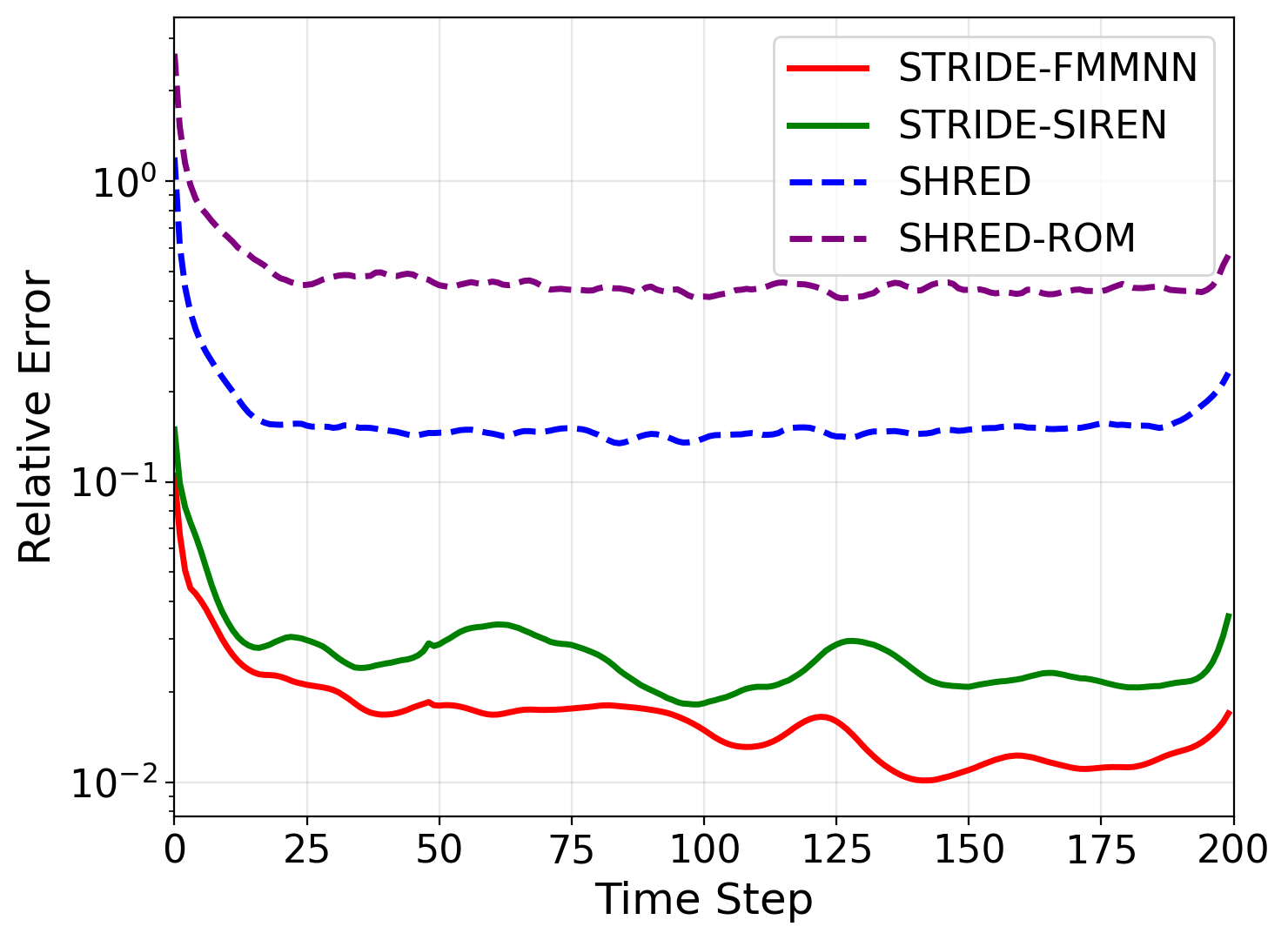}}
    \caption{Mean relative prediction error along test trajectories for the KS (left) and SWE (128$\times$128) (right) cases, comparing STRIDE-FMMNN and STRIDE-SIREN with baselines.}
    \label{fig:errplot}
  \end{center}
    \vskip -0.2in
\end{figure}

Fig.~\ref{fig:latent} shows representative latent state trajectories produced by STRIDE-FMMNN on FlowAO and Seismic. In FlowAO, the latent states exhibit clear periodic behavior after the initial burn-in (time lag $k$), consistent with periodic wake dynamics. In Seismic, the latent magnitude decays over time, matching the attenuation of the wavefield as it reaches the boundaries and is absorbed by the damping mask.

\begin{figure}[htb]
  \begin{center}
    \centerline{\includegraphics[width=0.45\columnwidth]{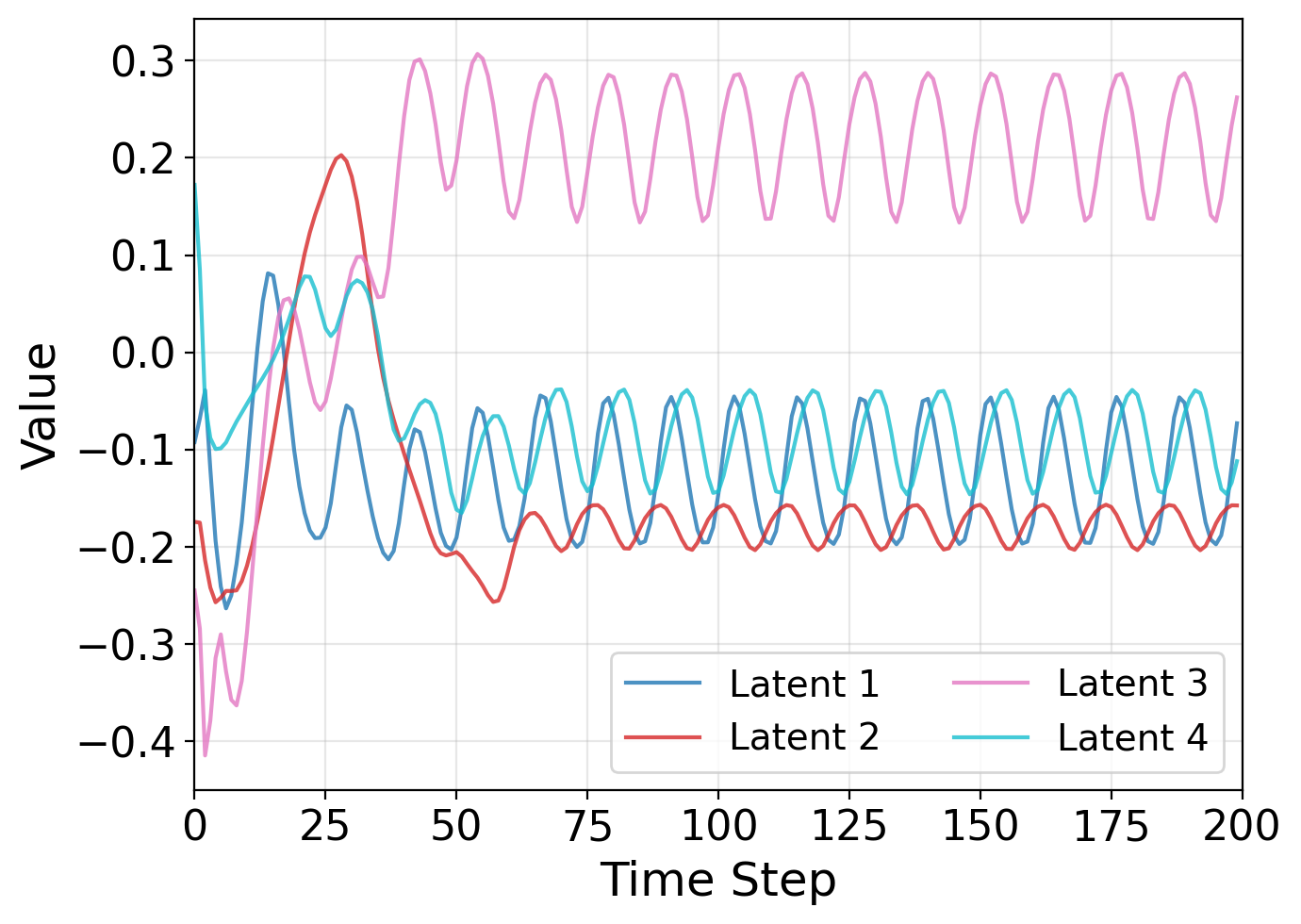}
    \includegraphics[width=0.45\columnwidth]{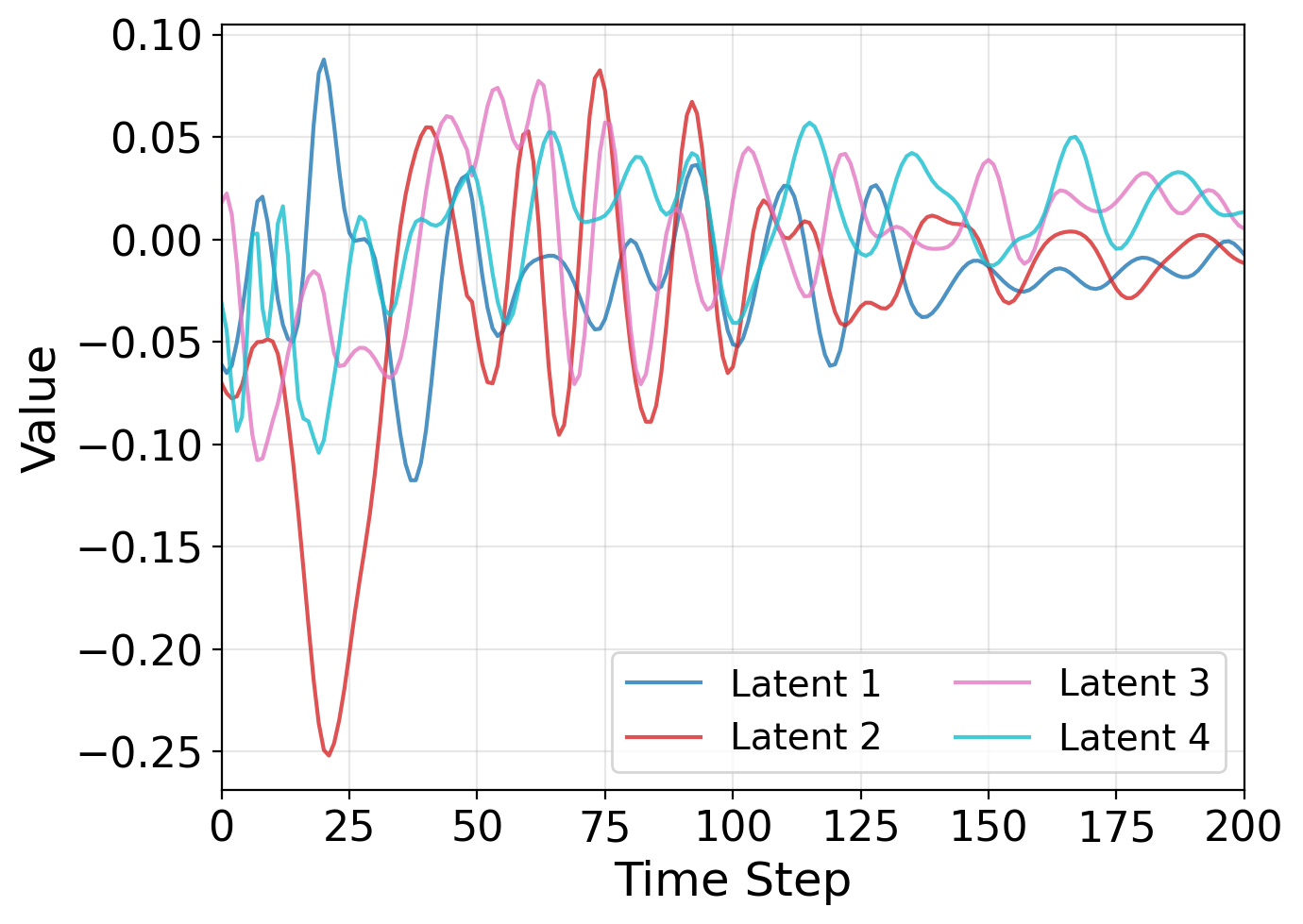}}
    \caption{The first four latent states of STRIDE-FMMNN in an example test trajectory for the FlowAO (left) and Seismic (right) cases.}
    \label{fig:latent}
  \end{center}
    \vskip -0.2in
\end{figure}

Beyond improved accuracy, STRIDE-FMMNN supports discretization- and resolution-invariant training via coordinate-based decoding and randomized spatial sampling. We consider a more challenging scenario where the standard deviation of the initial Gaussian bump is halved, effectively doubling the characteristic frequency of the wavefield. Here, we generate low-resolution data by downsampling the high-resolution simulation data, and arrange the sensors on a uniform grid to ensure consistent evaluation. As shown in Table~\ref{tab:resolutionhighfreq}, STRIDE-FMMNN achieves consistent performance across resolutions in standard settings, provided the model size and training budget (fixed at $N_\xi=4,096$ query points per snapshot) remain constant. In super-resolution settings (training at low resolution and testing at higher resolution), STRIDE-FMMNN maintains similar errors without retraining. Specifically, when super-resolving from 64 to 128, the relative errors are comparable to those obtained by training directly on the 128-resolution dataset. Furthermore, this approach significantly outperforms baseline methods that rely on bilinear interpolation to upsample low-resolution predictions.

\begin{table}[htb]
  \caption{Comparison of STRIDE-FMMNN prediction errors (\%) against interpolation baselines for the extended high-frequency SWE case on standard and super-resolution tasks.}
  \label{tab:resolutionhighfreq}
  \begin{center}
    \begin{small}
      \begin{sc}
        \begin{tabular}{c cc cc} 
            \toprule
            \multirow{2}{*}{Var.} & \multicolumn{2}{c}{Standard} & Super-Res & Interp. \\
            \cmidrule(lr){2-3} \cmidrule(lr){4-4} \cmidrule(lr){5-5}
            & 64 & 128 & 64 $\to$ 128 & 64 $\to$ 128 \\
            \midrule
            $u$      & 6.89 & 7.02 & 6.97 & 8.02 \\
            $v$      & 5.15 & 5.03 & 5.29 & 6.66 \\
            $\eta$   & 6.27 & 6.57 & 6.32 & 7.57 \\
            \bottomrule
        \end{tabular}
      \end{sc}
    \end{small}
  \end{center}
\end{table}

The FMMNN backbone in STRIDE-FMMNN yields more stable optimization and faster convergence than STRIDE-SIREN. Fig.~\ref{fig:trainingdyn} visualizes the training dynamics of three runs with STRIDE-FMMNN and STRIDE-SIREN under comparable model sizes and identical optimizer/scheduler settings. With a learning rate of 2e-3, STRIDE-FMMNN converges within 300 epochs, while STRIDE-SIREN converges (by validation loss) in about 700 epochs with a learning rate of 1e-4. With a larger learning rate for STRIDE-SIREN (e.g., 2e-4), we observe an early spike that leads to non-convergence. Finally, STRIDE-FMMNN exhibits higher training loss but lower validation loss at convergence, suggesting reduced overfitting due to its low-rank, partially trainable parameterization.

\begin{figure}[htb]
  \begin{center}
    \centerline{\includegraphics[width=0.32\columnwidth]{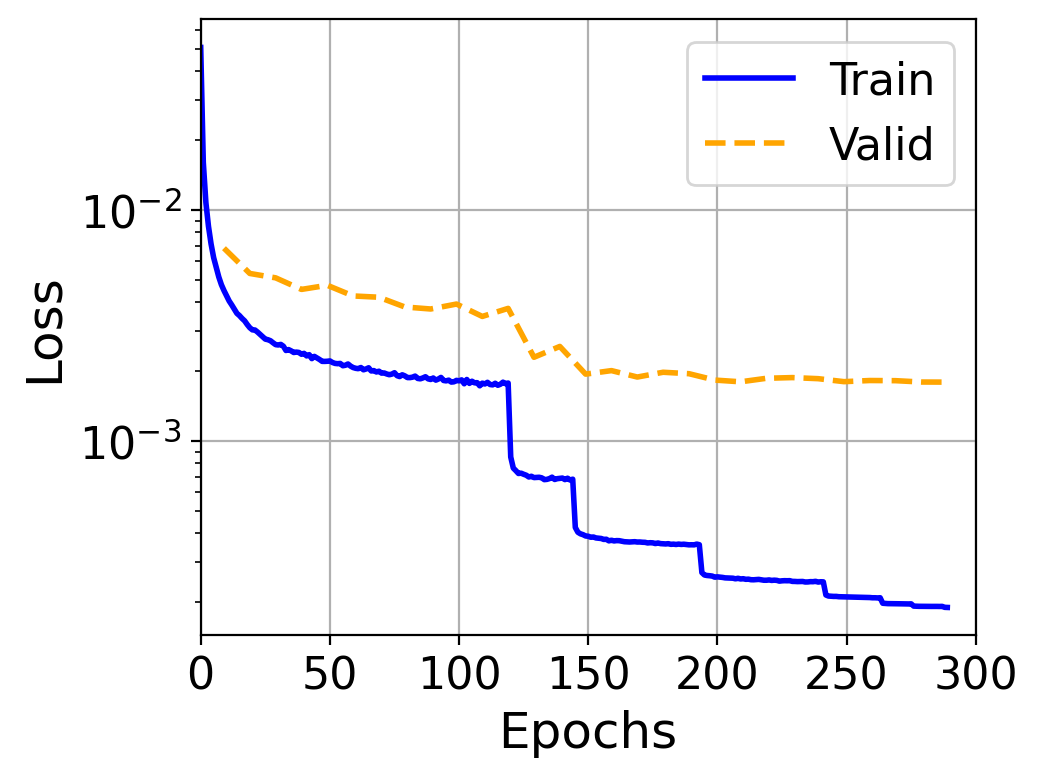}
    \includegraphics[width=0.32\columnwidth]{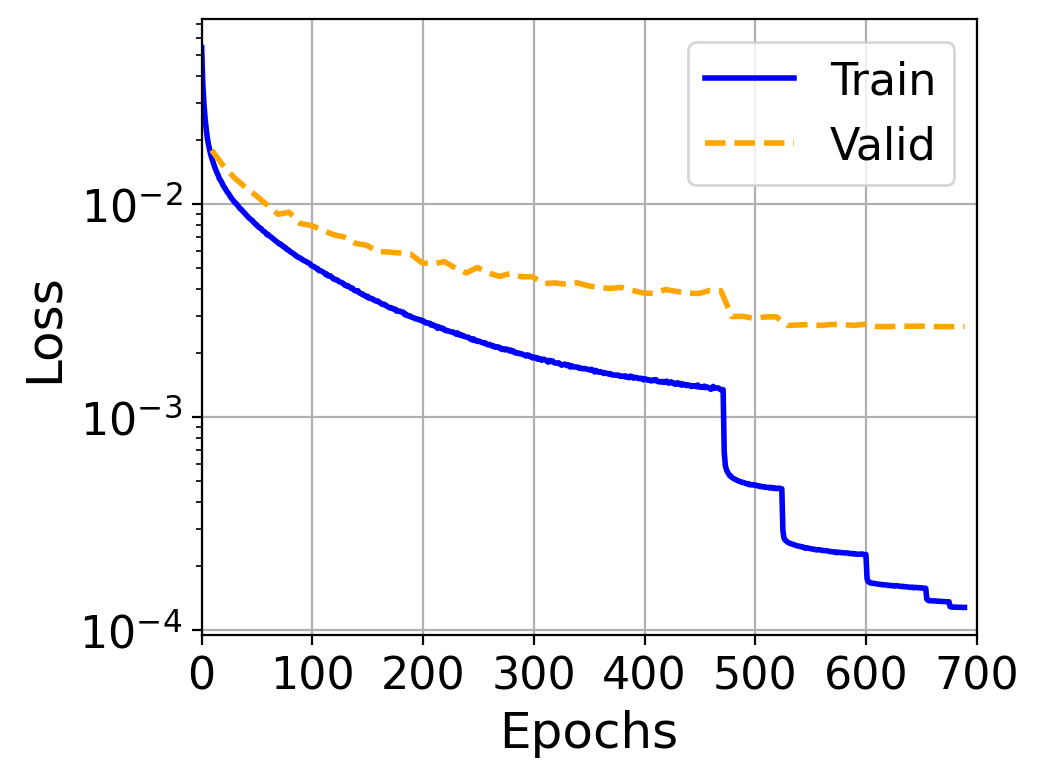}
    \includegraphics[width=0.32\columnwidth]{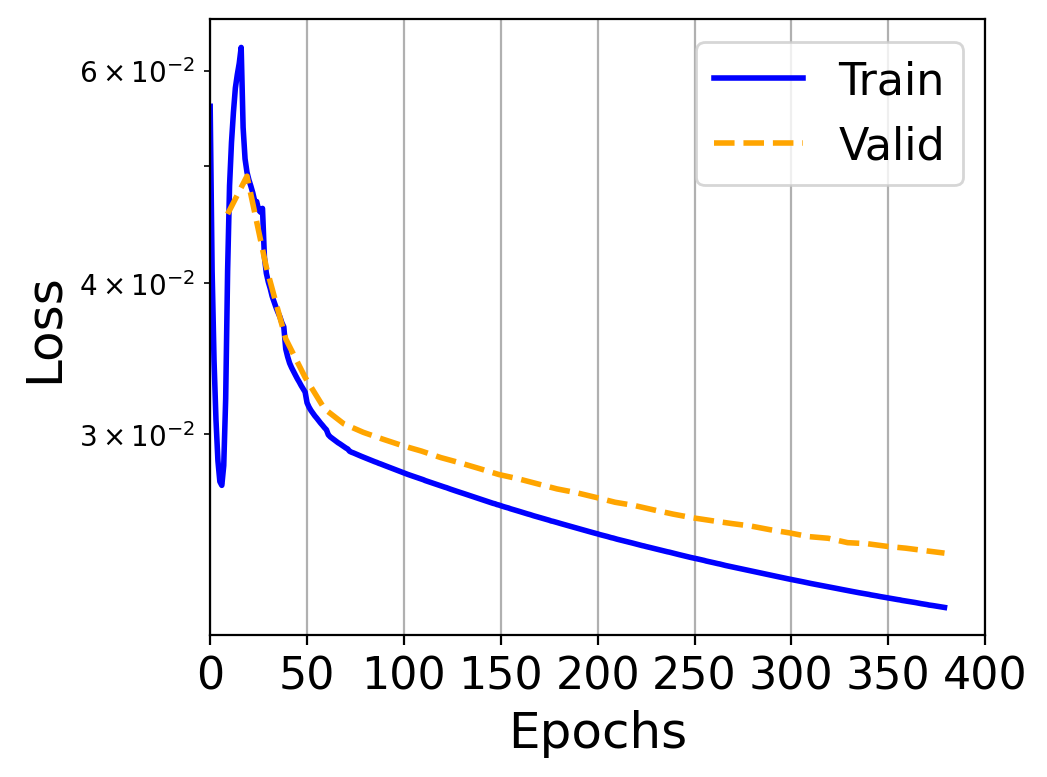}}
    \caption{Training dynamics of STRIDE-FMMNN (left) and STRIDE-SIREN (middle and right) for the KS case. The three trainings use the same optimizer and scheduler, with learning rates set to 2e-3, 1e-4, and 2e-4, respectively.}
    \label{fig:trainingdyn}
  \end{center}
    \vskip -0.2in
\end{figure}

Fig.~\ref{fig:ablation} shows the impact of time lag and number of random sensors on the overall relative error of STRIDE-FMMNN on the SWE (128$\times$128) dataset. Empirically, the relative error decreases with longer time lags and more random sensors. This validates the choice of time lag and number of sensors as discussed in Section~\ref{sec:theory}.

\begin{figure}[htb]
  \begin{center}
    \centerline{\includegraphics[width=0.45\columnwidth]{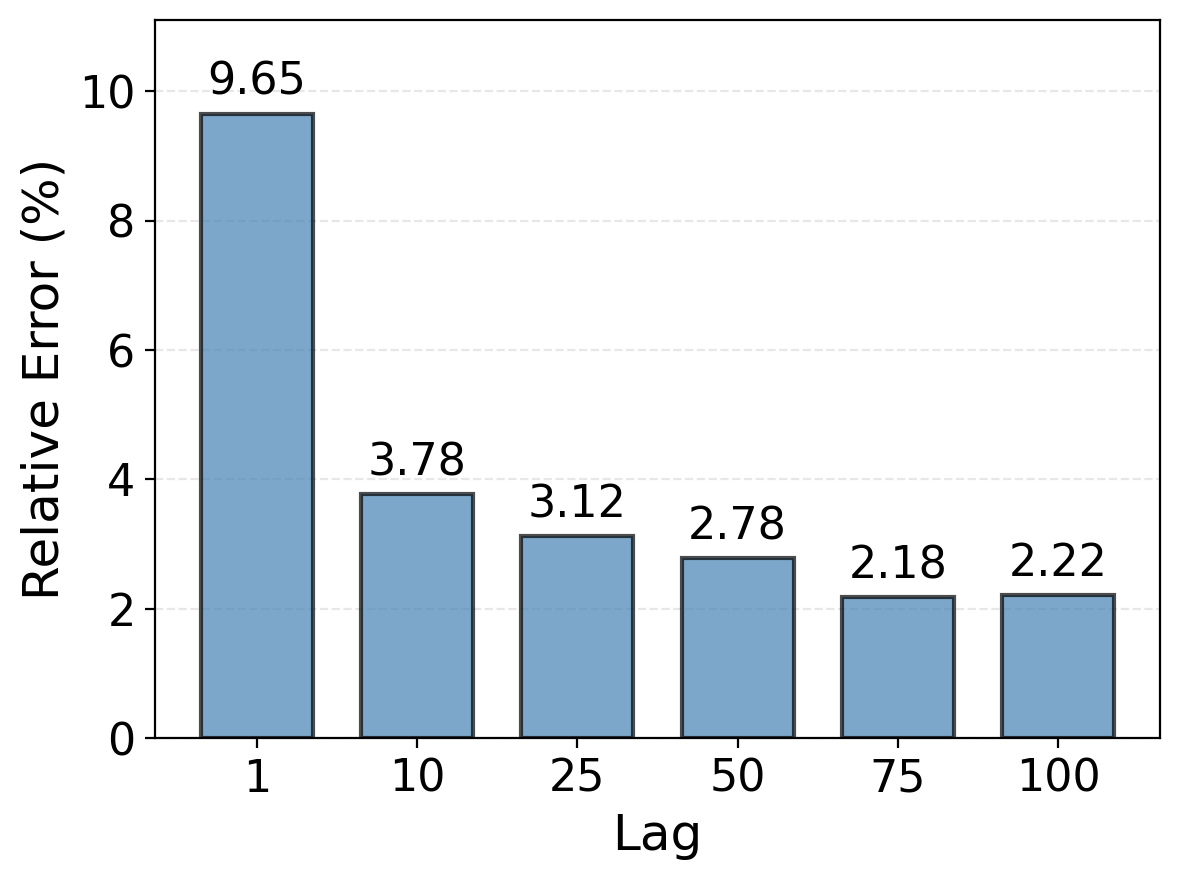}
    \includegraphics[width=0.45\columnwidth]{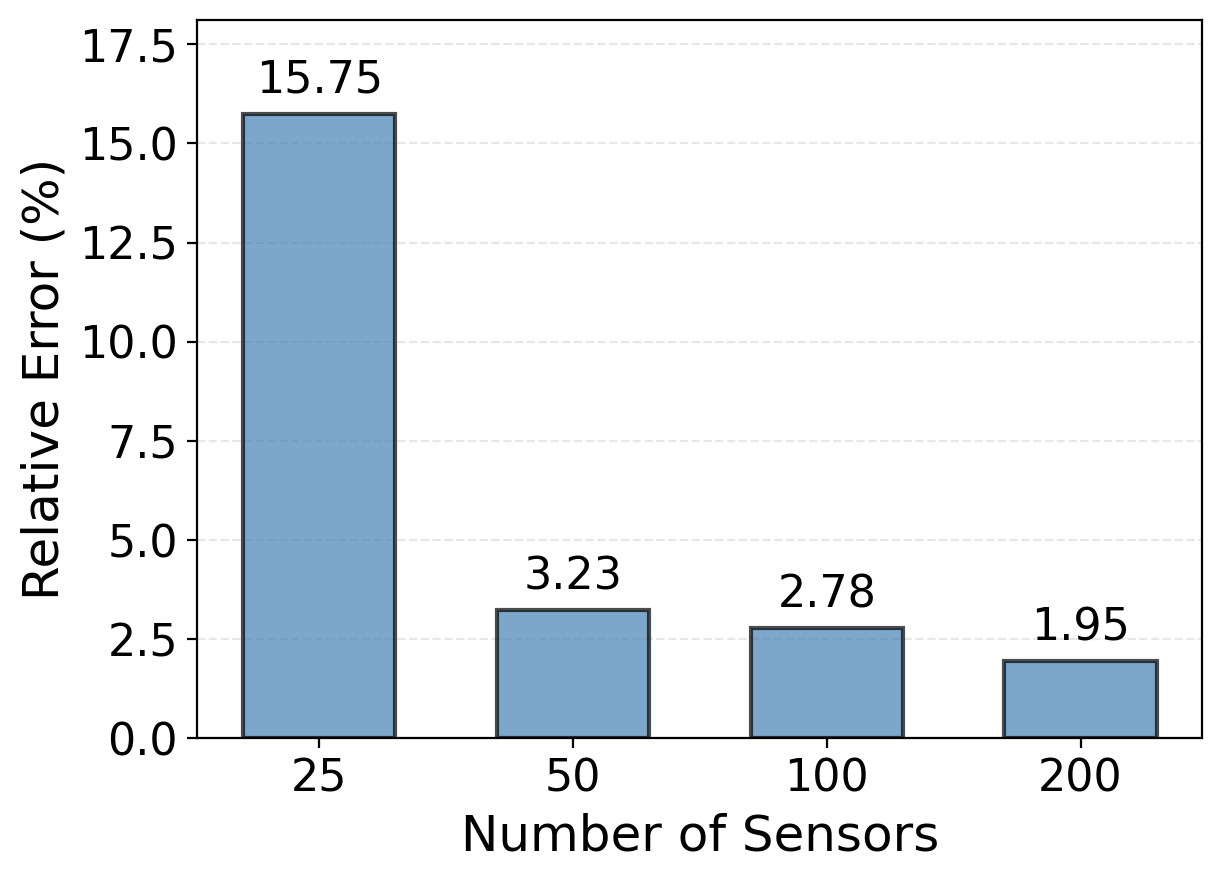}
    }
     \caption{Ablation studies of STRIDE-FMMNN across time lags and the number of sensors on the SWE (128$\times$128) dataset.}
    \label{fig:ablation}
  \end{center}
  \vskip -0.2in
\end{figure}

Table~\ref{tab:noise} evaluates the robustness of STRIDE-FMMNN against different observation noise levels. We define the noise level as the standard deviation of the applied Gaussian noise relative to the standard deviation of the true observations. The results demonstrate the framework's robustness, yielding a remarkably low prediction error of 6.24\% under a high noise condition of 20\%. 

\begin{table}[htb]
  \caption{STRIDE-FMMNN prediction errors (\%) for $\eta$ with different observation noise levels (\%)}
  \label{tab:noise}
  \begin{center}
    \begin{small}
      \begin{sc}
        \begin{tabular}{c cccc} 
            \toprule
            Noise level & 0 & 5 & 10 & 20 \\
            \midrule
            Error     & 2.78   & 3.22    & 4.04   & 6.24 \\
            \bottomrule
        \end{tabular}
      \end{sc}
    \end{small}
  \end{center}
\end{table}

We compare different temporal encoders within STRIDE-FMMNN, using a fixed latent dimension and the same FMMNN decoder. The left panel of Fig.~\ref{fig:Mamba} shows that recurrent models (LSTM, GRU, and Mamba) substantially outperform MLP and self-attention variants in both accuracy and parameter efficiency. Mamba achieves the lowest prediction error with a model size comparable to LSTM. The right panel summarizes LSTM vs.\ Mamba across datasets: Mamba slightly outperforms LSTM in most cases with KS as an exception. This suggests that Mamba is a promising temporal encoder for extracting latent states.

\begin{figure}[htb]
  \begin{center}
    \centerline{\includegraphics[width=0.45\textwidth]{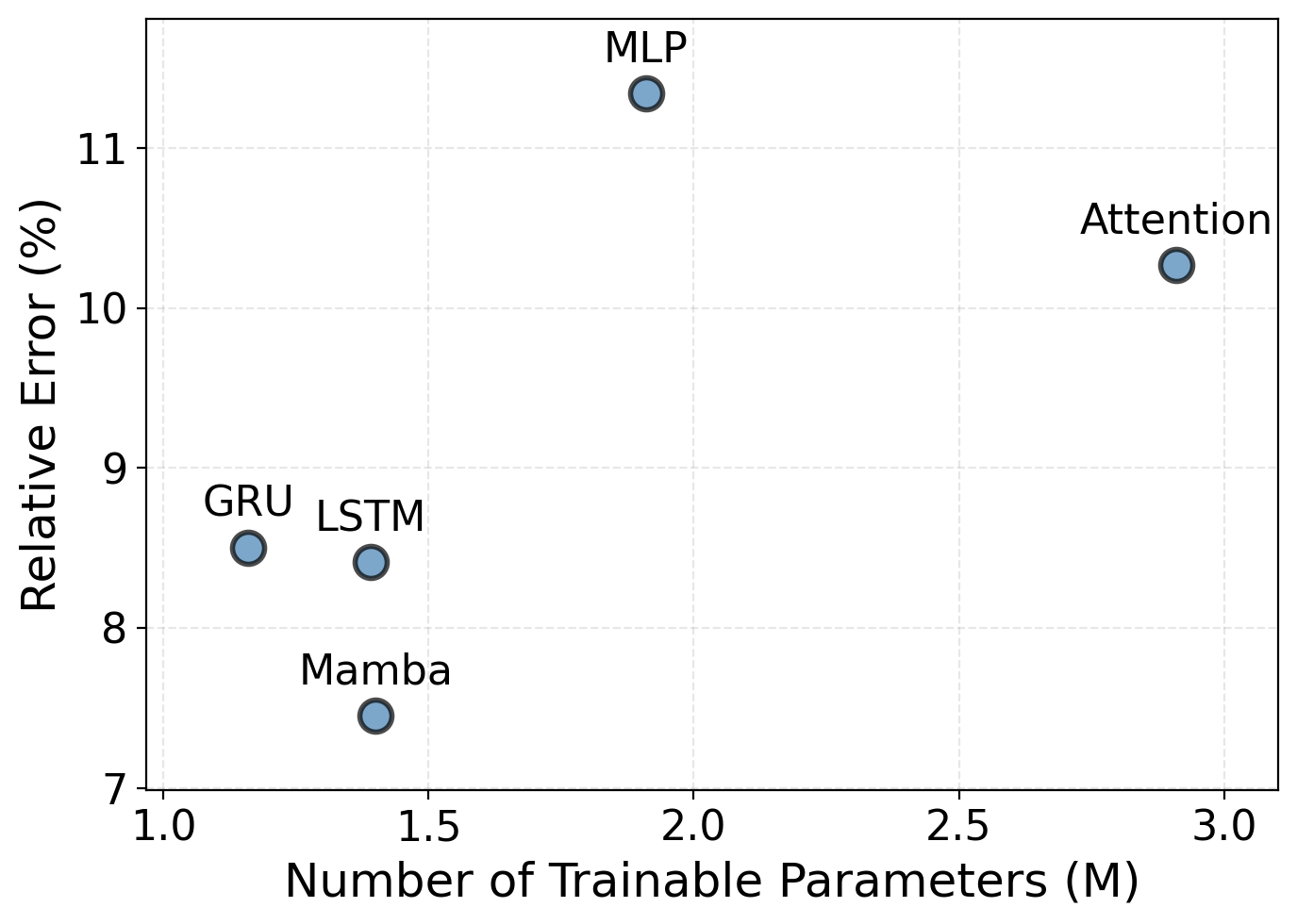}
    \includegraphics[width=0.45\textwidth]{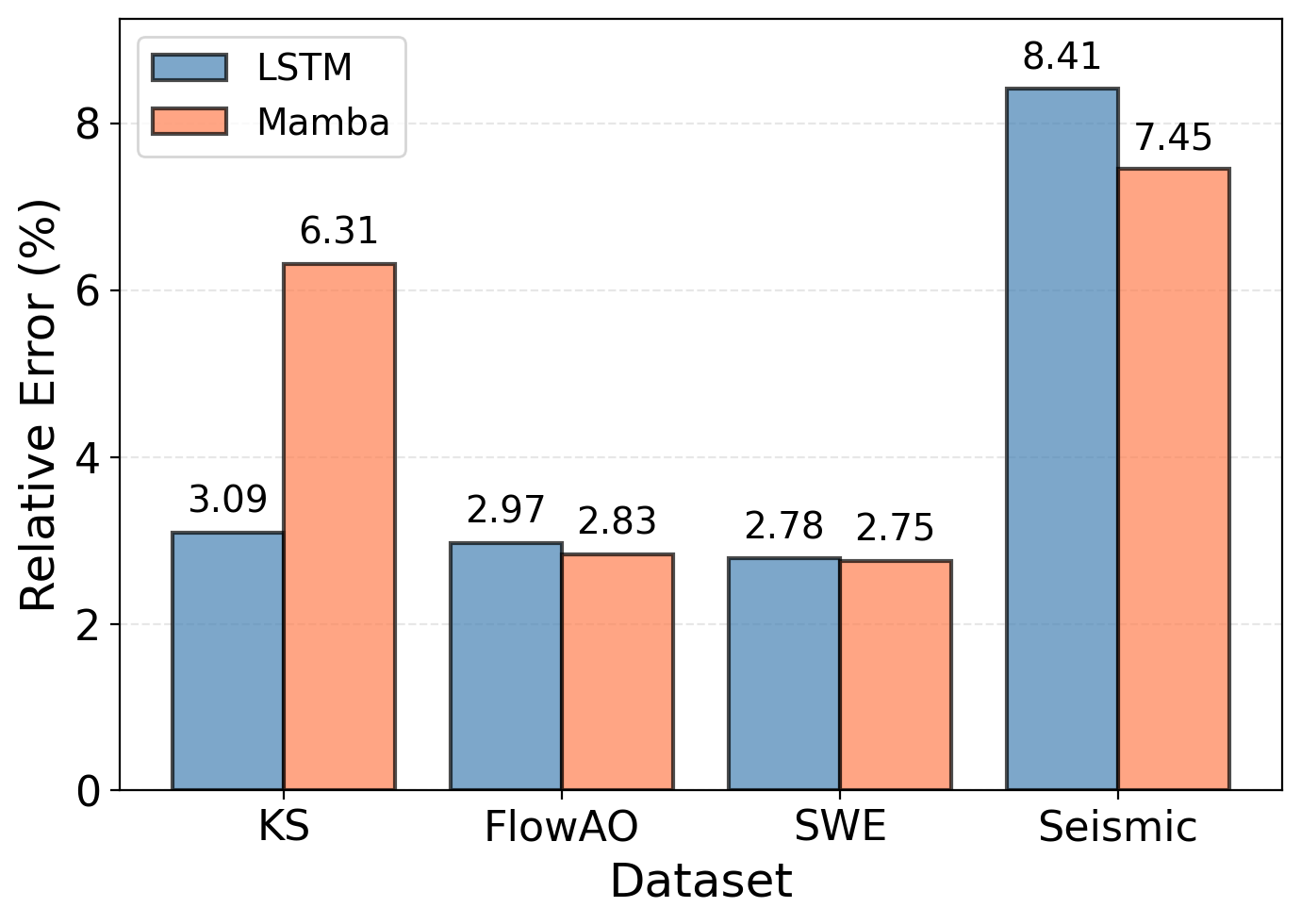}}
    \caption{Comparison of temporal models on the Seismic dataset (left) and prediction errors using LSTM or Mamba per dataset (right).}
    \label{fig:Mamba}
  \end{center}
  \vskip -0.2in
\end{figure}

Fig.~\ref{fig:ood} demonstrates the performance of STRIDE-FMMNN on both in-distribution and out-of-distribution (OOD) samples for the SWE case. When test samples are drawn within the training range, the proposed model achieves very low relative errors as previously reported. However, generalizing to unseen scenarios remains a significant challenge, even when the dynamics are simply flipped in the OOD samples. Prediction errors remain below 10\% only when the test samples are fairly close to the training range. For samples further outside this distribution, relative errors easily exceed 30\%.

\begin{figure}[htb]
  \begin{center}
    \centerline{\includegraphics[width=0.45\textwidth]{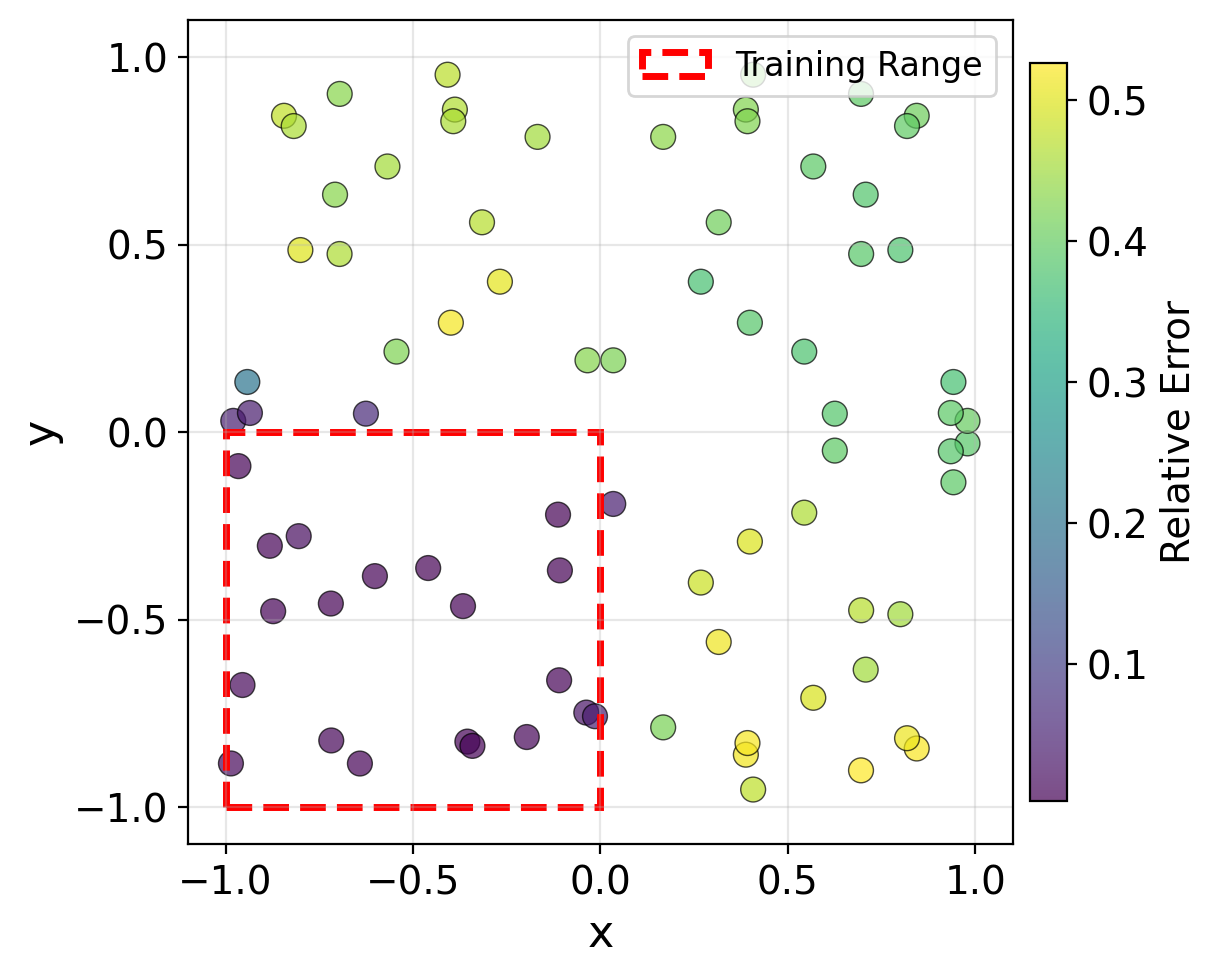}}
    \caption{Generalization analysis of STRIDE-FMMNN on the SWE dataset. The scatter plot displays the relative prediction error for surface elevation ($\eta$) across in-distribution and out-of-distribution samples. The coordinates of each point correspond to the center of the initial Gaussian bump. Training trajectories are exclusively initialized within the lower-left quarter.}
    \label{fig:ood}
  \end{center}
  \vskip -0.2in
\end{figure}

In addition to full-field reconstruction, the underlying control parameters can be accurately estimated from the informative latent states learned by STRIDE-FMMNN using a simple MLP. As illustrated in Fig.~\ref{fig:estimation}, the estimated parameters for the SWE and FlowAO cases closely match the ground truth in an example test trajectory. Among all the test trajectories, the mean absolute errors for the two parameters defining the location of the initial Gaussian bump (sampled from $[-1, 0] \times [-1, 0]$) are \textbf{1.47e-3} and \textbf{8.33e-4}. The mean absolute error of the time-varying angle of attack is \textbf{2.10e-2}.

\begin{figure*}[htb]
  \begin{center}
    \centerline{\includegraphics[width=0.33\textwidth]{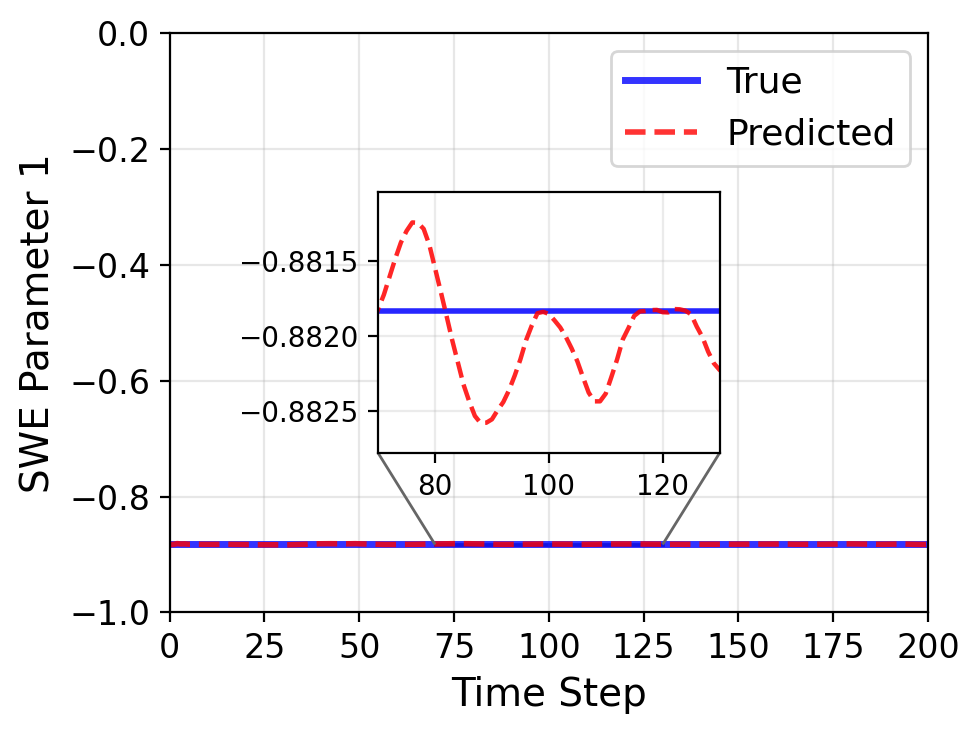}
    \includegraphics[width=0.33\textwidth]{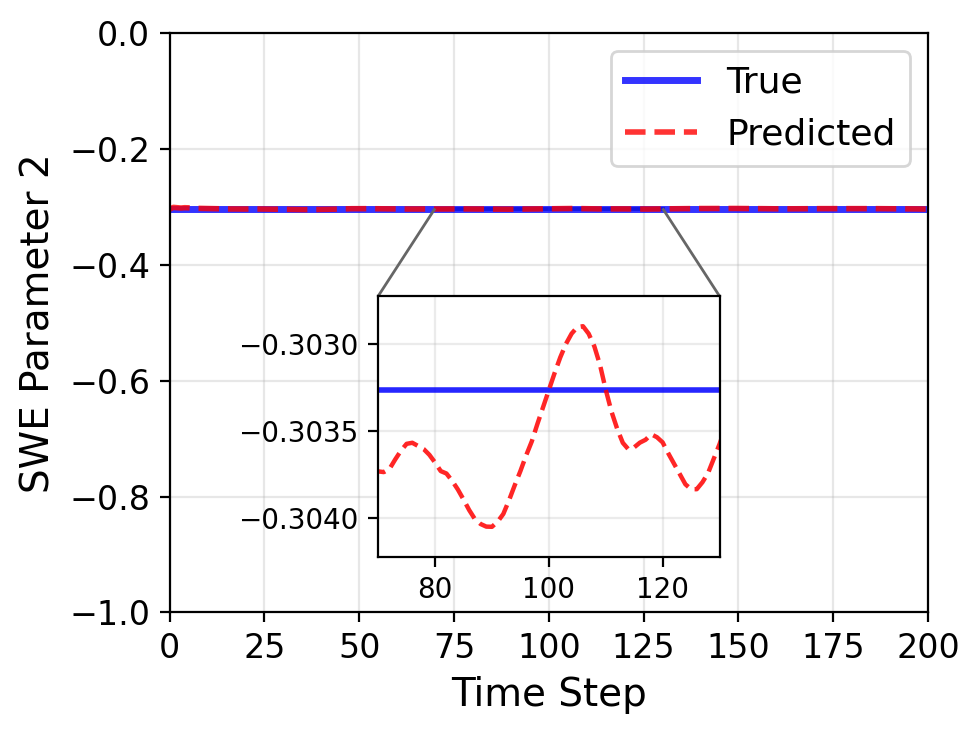}
    \includegraphics[width=0.33\textwidth]{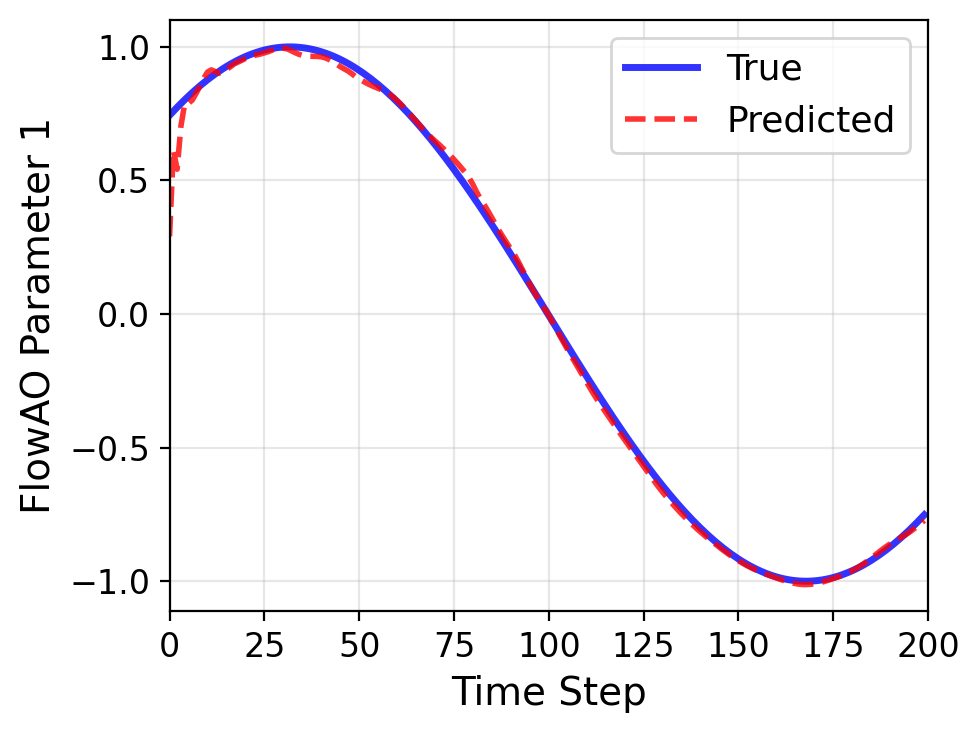}}
    \caption{Parameter estimation from STRIDE-FMMNN learned latent states. Left and Middle: Estimation of the initial Gaussian bump coordinates ($x$- and $y$-locations) in an example test trajectory for the SWE case. Right: Estimation of the time-varying angle of attack for the FlowAO case.}
    \label{fig:estimation}
  \end{center}
  \vskip -0.2in
\end{figure*}

Fig.~\ref{fig:forecasting} shows the performance of auto-regressive forecasting using STRIDE-FMMNN. To extend the framework for future-state prediction, we train an auxiliary LSTM network on the time series of learned latent states. This enables the extrapolation of future latent trajectories, which are subsequently mapped back to full spatial fields via the pre-trained FMMNN decoder. In the SWE case, the model demonstrates robust short-term forecasting capabilities, maintaining relative errors of approximately 25\% across all three state variables after a horizon of 21 time steps. However, the Seismic case proves significantly more challenging. The forecasting error blows up quickly and approaches 100\% over the same 21-step interval.

\begin{figure}[htb]
  \begin{center}
    \centerline{\includegraphics[width=0.45\textwidth]{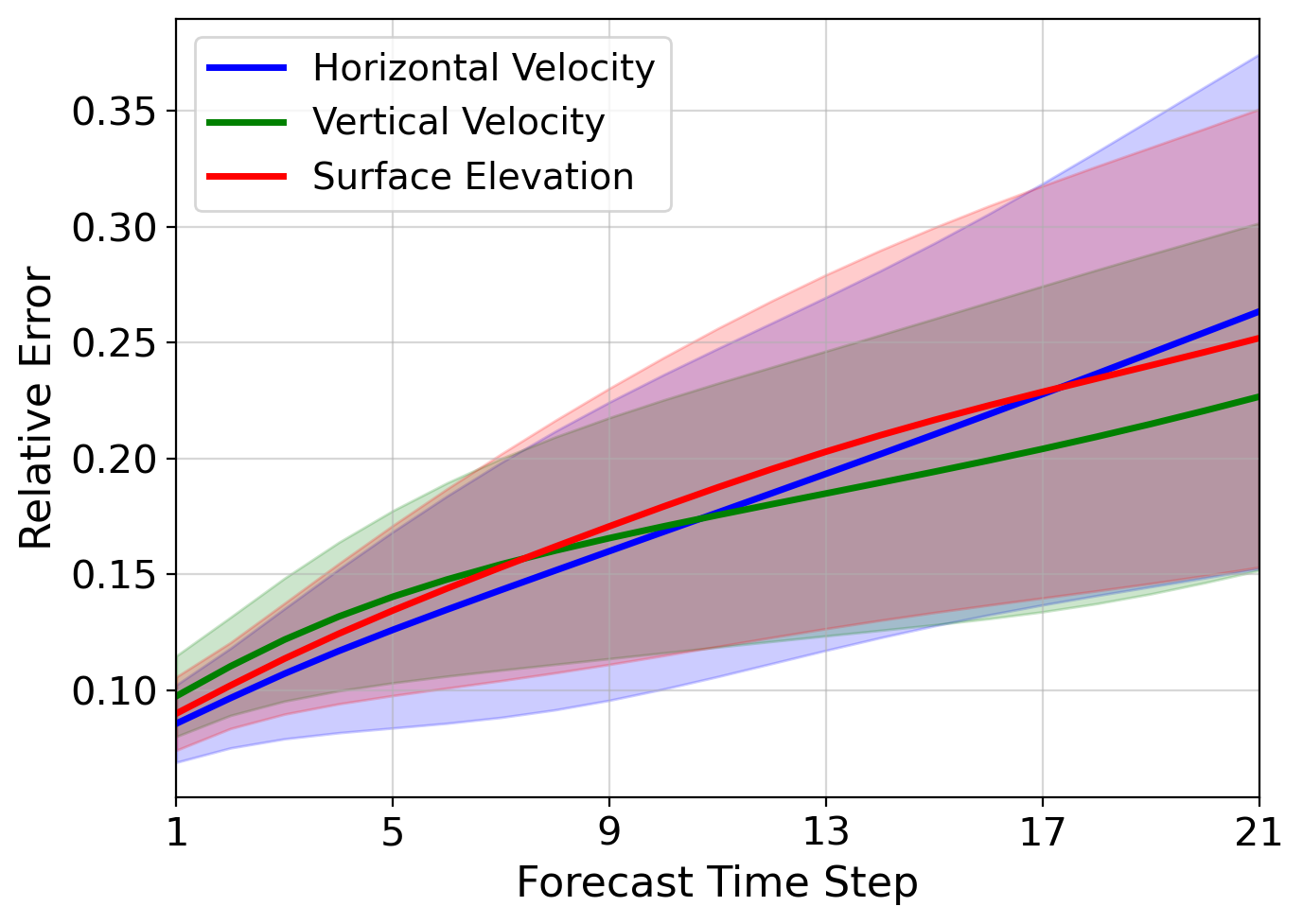}
    \includegraphics[width=0.45\textwidth]{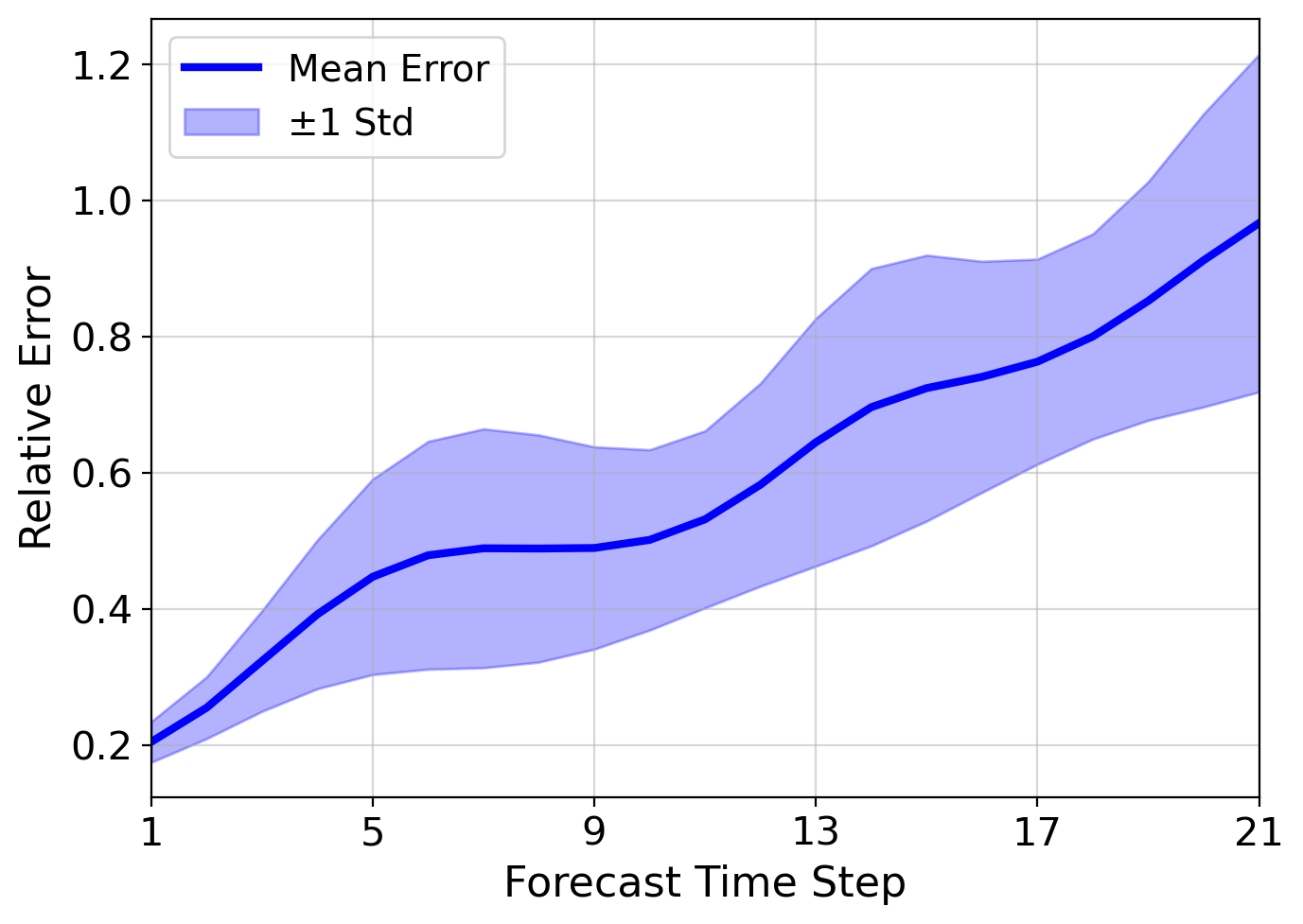}}
    \caption{Auto-regressive forecasting performance of STRIDE-FMMNN. The plots illustrate the temporal evolution of relative prediction error for the SWE (left) and Seismic (right) cases. The shaded regions denote the standard deviation across all test trajectories.}
    \label{fig:forecasting}
  \end{center}
  \vskip -0.2in
\end{figure}

The performance of the proposed framework also depends on sensor placement, especially when observations are extremely sparse. As shown in Fig.~\ref{fig:placement}, we trained STRIDE-FMMNN for the SWE case with 25 sensors arranged according to three different strategies. Here, QR refers to the SSPOR algorithm implemented in PySensors \cite{karnik_pysensors_2025}, which utilizes QR optimization by default. The three configurations yielded overall relative errors for $\eta$ of \textbf{15.75\%}, \textbf{16.28\%}, and \textbf{6.49\%}, respectively. These findings underscore the critical role of data-driven sensor placement in minimizing reconstruction error. Notably, QR placement achieves the lowest error by concentrating sensors in the lower left quarter, where the Gaussian bump is initialized, thereby capturing more information about the underlying dynamics.

\begin{figure}[htb]
  \begin{center}
    \centerline{\includegraphics[width=0.85\textwidth]{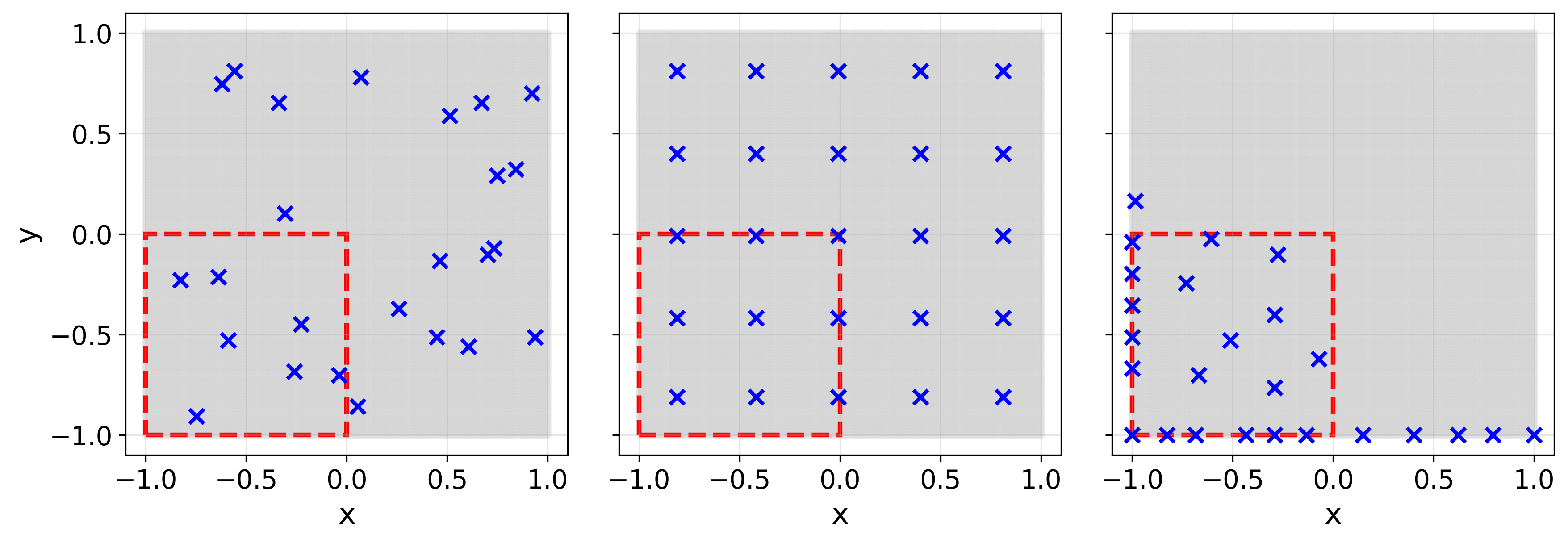}}
    \caption{Comparison of sparse sensor placement strategies in the SWE domain: random (left), uniform (middle), and QR (right). The relative prediction errors for $\eta$ with the three placement strategies are 15.75\%, 16.28\%, and 6.49\%, respectively.}
    \label{fig:placement}
  \end{center}
  \vskip -0.2in
\end{figure}

\subsection{Further Discussion}
STRIDE inherits the computational trade-offs of INR-based decoders. Training requires backpropagation through a comparatively deep coordinate-to-field network and repeated evaluation of the decoder at many query coordinates (even when we subsample spatial points during training). At inference time, reconstructing a full field requires querying the decoder at all spatial points, so the runtime scales (approximately) linearly with the output resolution; in contrast, SHRED and SHRED-ROM output the full field (or low-dimensional coefficients) with a single forward pass. Table~\ref{tab:time} shows that both STRIDE-FMMNN and STRIDE-SIREN incur substantially higher evaluation costs than the non-INR baselines, and they can also be slower to train depending on the spatial resolution and the decoder backbone. Training and evaluation are performed on a single NVIDIA L40S GPU.

\begin{table}[htb]
  \caption{Model comparison of training and evaluation time per dataset. STRIDE-FMMNN uses the FMMNN INR backbone; STRIDE-SIREN uses a SIREN backbone.}
  \label{tab:time}
  \begin{center}
    \begin{small}
      \begin{sc}
        \begin{tabular}{cc cccc} 
            \toprule
            Model & dataset $\rightarrow$ & \emph{KS} & \emph{FlowAO} & \emph{SWE (128$\times$128)} & \emph{Seismic} \\
            \midrule
            \multirow{2}{*}{SHRED} & $T_{train}$ (s) & 1178.20 & 3502.07 & 2969.87 & 14356.15 \\
                         & $T_{test}$ (s)       & 0.15 & 0.64 & 0.45 & 0.35 \\
            \midrule
            \multirow{2}{*}{SHRED-ROM} & $T_{train}$ (s) & 2642.42 & 1549.72 & 2454.69 & 1293.78 \\
                         & $T_{test}$ (s)        & 0.14 & 0.10 & 0.12 & 0.25 \\
            \midrule
            \multirow{2}{*}{\shortstack{STRIDE-\\SIREN}} & $T_{train}$ (s) & 4993.41 & 4335.55 & 5140.88 & 64340.34 \\
            	                         & $T_{test}$ (s)        & 3.30 & 16.06 & 14.19 & 6.90 \\
            \midrule
            \multirow{2}{*}{\shortstack{STRIDE-\\FMMNN}} & $T_{train}$ (s) & 2076.49 & 3895.35 & 7203.90 & 97803.72 \\
            	                         & $T_{test}$ (s)        & 3.35 & 13.93 & 17.82 & 14.57 \\
            \bottomrule
        \end{tabular}
      \end{sc}
    \end{small}
  \end{center}
\end{table}

These costs are the price of resolution-invariant reconstruction and improved expressivity. When super-resolution is not needed, the decoder can be evaluated on a coarser grid to trade accuracy for speed; more generally, future work could focus on accelerating INR inference (e.g., distillation to a lightweight grid decoder, hierarchical/coarse-to-fine querying, or caching on regular grids). Moreover, it remains to evaluate the framework on large-scale practical problems, particularly those characterized by higher-frequency components.

\section{Conclusion}
\label{sec:conc}
We introduced STRIDE, a two-stage framework for reconstructing spatiotemporal fields from sparse point-sensor measurements. STRIDE first maps a short observation window to a low-dimensional latent state using a temporal encoder, and then decodes the latent state with a coordinate-based implicit neural representation (INR) to produce continuous, resolution-invariant reconstructions. Across four benchmarks spanning chaotic dynamics and wave propagation, our default instantiation STRIDE-FMMNN achieves lower reconstruction errors than strong baselines and is more stable to train than the SIREN-backbone variant STRIDE-SIREN, while retaining the ability to super-resolve by querying the decoder at arbitrary spatial coordinates.

From a theoretical standpoint, we provided a justification for why a finite-dimensional latent state can suffice: under delay observability of point measurements on the relevant invariant set, the reconstruction operator factors through a finite-dimensional embedding, making STRIDE-type architectures natural approximators. A key limitation is the computational cost inherited from INR decoders, since reconstructing a full field requires many coordinate queries. Future work will focus on accelerating INR inference and scaling STRIDE to larger, higher-frequency problems and more challenging sensing regimes.

\section*{Acknowledgment:} The authors thank Prof. Nathan Kutz for the motivating discussion during his visit to Georgia Tech, and Sebastian Gutierrez Hernandez for his group presentation on the FMMNN architecture \cite{zhang_fourier_2025}. This research was supported by NSF grant \# 2325631.

\bibliographystyle{unsrt}  
\bibliography{references}  

\newpage
\appendix
\onecolumn

\section{Proof}

\subsection{Proof of Lemma~\ref{lem:stability-T}}
\label{sec:proof-lemma}
\begin{proof}
Let $\iota:X\hookrightarrow C^0(\Omega,\mathbb{R}^{d_x})$ denote the continuous embedding, and define $G_k:\mathcal{A}\to X$ by $G_k(\bsmu,\bsx):=F_\bsmu^k\bsx$. Then $T(\bsy)=\iota\!\left(G_k(\Psi(\bsy))\right)$.
Under Assumption~\ref{assump:delay-obs}, $\Psi:Y\to\mathcal{A}$ is Lipschitz and hence continuous. By assumption, $G_k$ is continuous on $\mathcal{A}$. Therefore, the composition $\bsy\mapsto \iota(G_k(\Psi(\bsy)))$ is continuous as a map from $Y$ into $C^0(\Omega,\mathbb{R}^{d_x})$, which proves the continuity of $T$.
For the stability bound, for any $\bsy,\bsy'\in Y$ we have
\[
\|T(\bsy)-T(\bsy')\|_{C^0}
\le C_{\mathrm{ev}}\|G_k(\Psi(\bsy))-G_k(\Psi(\bsy'))\|_X,
\]
where we used the embedding inequality. Using that $G_k$ is $L_{F^k}$-Lipschitz on $\mathcal{A}$ and $\Psi$ is $L_\Psi$-Lipschitz on $Y$,
\begin{align*}
    \|G_k(\Psi(\bsy))-G_k(\Psi(\bsy'))\|_X
    &\le L_{F^k}\|\Psi(\bsy)-\Psi(\bsy')\|_{\mathcal{P}\times X} \le L_{F^k}L_\Psi\|\bsy-\bsy'\|.
\end{align*}
Combining the two displays yields the claimed inequality.
\end{proof}

\subsection{Proof of Theorem~\ref{thm:operator-approx}}
\label{sec:proof-theorem}
\begin{proof}
Let $q=(k+1)p$ and recall that $Y=\Phi_k(\mathcal{A})\subset\mathbb{R}^{q}$ is compact. Define the end-of-window state map $U:Y\to X$ by
\[
U(\bsy):=G_k(\Psi(\bsy))\in X,\quad \bsy\in Y.
\]
Since $\Psi$ is continuous on $Y$ (Assumption~\ref{assump:delay-obs}) and $G_k$ is continuous on $\mathcal{A}$ by assumption, $U$ is continuous. Moreover, $U(Y)\subset \mathcal{A}_X$.

Because $\mathcal{A}_X$ is compact with finite box-counting dimension $d_X$, the Ma\~n\'e projection theorem implies that for any integer ${d_z}>2d_X$ there exists a bounded linear map $E:X\to\mathbb{R}^{d_z}$ that is injective on $\mathcal{A}_X$ and admits a (Hölder, hence continuous) inverse $E^{-1}:Z\to\mathcal{A}_X$ on $Z:=E(\mathcal{A}_X)\subset\mathbb{R}^{d_z}$ \cite{mane1981dimension}. Define the latent map $H:Y\to Z$ by $H(\bsy):=E(U(\bsy))$, which is continuous as a composition.
Next define the decoder target $D:\Omega\times Z\to\mathbb{R}^{d_x}$ by
\[
D(\bsxi,\bsz):=\bigl(E^{-1}(\bsz)\bigr)(\bsxi).
\]
This is well-defined because $E^{-1}(\bsz)\in\mathcal{A}_X\subset X$ and point evaluation is continuous on $X$. Since $E^{-1}$ is continuous on $Z$ and $(\bsxi,\bsx)\mapsto \bsx(\bsxi)$ is continuous on $\Omega\times X$, $D$ is continuous on the compact set $\Omega\times Z$. Finally, for any $\bsy\in Y$ and $\bsxi\in\Omega$,
\[
T(\bsy)(\bsxi)=U(\bsy)(\bsxi)=\bigl(E^{-1}(H(\bsy))\bigr)(\bsxi)=D\bigl(\bsxi,H(\bsy)\bigr),
\]
which proves the factorization.

For approximation, note that $H$ is continuous on compact $Y$ and $D$ is continuous on compact $\Omega\times Z$. Hence $D$ is uniformly continuous; let $\omega_D$ be a modulus of continuity in its second argument:
\begin{equation*}
    \omega_D(r):=\sup\bigl\{\|D(\bsxi,\bsz)-D(\bsxi,\bsz')\|:\, \bsxi\in\Omega,\ \bsz,\bsz'\in Z,\ \|\bsz-\bsz'\|\le r\bigr\},
\end{equation*}
so $\omega_D(r)\to 0$ as $r\to 0$. Fix $\varepsilon>0$ and choose $\delta>0$ such that $\omega_D(\delta)\le \varepsilon/2$.
By universal approximation on compact sets, there exist neural-network parametrizations $\mathcal{G}:Y\to\mathbb{R}^{d_z}$ and $\mathcal{F}:\Omega\times\mathbb{R}^{d_z}\to\mathbb{R}^{d_x}$ such that
\begin{equation*}
    \sup_{\bsy\in Y}\|H(\bsy)-\mathcal{G}(\bsy)\|\le \delta
    \quad\text{and }
    \sup_{\bsxi\in\Omega,\ \bsz\in Z}\|D(\bsxi,\bsz)-\mathcal{F}(\bsxi,\bsz)\|\le \varepsilon/2,
\end{equation*}
for example using recurrent/window encoders for $\mathcal{G}$ and feedforward/modulated networks for $\mathcal{F}$ \cite{funahashi1993approximation,hornik1989multilayer,leshno1993multilayer}.
Then, for any $\bsy\in Y$ and $\bsxi\in\Omega$, 
\begin{align*}
&\bigl\|T(\bsy)(\bsxi)-\mathcal{F}(\bsxi,\mathcal{G}(\bsy))\bigr\|
=\bigl\|D(\bsxi,H(\bsy))-\mathcal{F}(\bsxi,\mathcal{G}(\bsy))\bigr\|\\
&\le \|D(\bsxi,H(\bsy))-D(\bsxi,\mathcal{G}(\bsy))\|+\|D(\bsxi,\mathcal{G}(\bsy))-\mathcal{F}(\bsxi,\mathcal{G}(\bsy))\|\\
&\le \omega_D(\|H(\bsy)-\mathcal{G}(\bsy)\|)+\varepsilon/2
\le \omega_D(\delta)+\varepsilon/2\le \varepsilon,
\end{align*}
which concludes the proof.
\end{proof}

\begin{remark}[Approximate embedding and an error floor]\label{rem:approx-embedding}
Assumption~\ref{assump:delay-obs} is an idealization: with finite sampling, sensor noise, and limited sensor budgets, the delay map $\Phi_k$ may fail to be exactly injective on $\mathcal{A}$, or it may be injective but poorly conditioned.
Two standard relaxations make the stability picture explicit.
\emph{(i) Injective but not Lipschitz.} If $\Phi_k$ is injective on $\mathcal{A}$ but the inverse is only uniformly continuous, one may replace the Lipschitz bound in Assumption~\ref{assump:delay-obs} by a \emph{modulus of continuity} $\omega$ such that
\begin{equation}
    \|\Psi(\bsy)-\Psi(\bsy')\|_{\mathcal{P}\times X}\le \omega(\|\bsy-\bsy'\|),\quad \bsy,\bsy'\in Y.
\end{equation}
Then the stability estimate becomes
\begin{equation}
    \|T(\bsy)-T(\bsy')\|_{C^0}\le C_{\mathrm{ev}}L_{F^k}\,\omega(\|\bsy-\bsy'\|),
\end{equation}
so observation perturbations translate into reconstruction perturbations according to the (possibly sub-Lipschitz) conditioning encoded by $\omega$.
\emph{(ii) Non-injective (intrinsic ambiguity).} If $\Phi_k$ is not injective, a useful identifiability proxy is the ambiguity function
\begin{equation}
    \alpha(\varepsilon):=\sup\Bigl\{\|G_k(a)-G_k(a')\|_X:\ a,a'\in\mathcal{A},
    \ \|\Phi_k(a)-\Phi_k(a')\|\le \varepsilon\Bigr\},
\end{equation}
where $a=(\bsmu,\bsx)$ denotes the augmented state and $G_k(\bsmu,\bsx):=F_\bsmu^k\bsx$ is the end-of-window state.
When $\alpha(\varepsilon)>0$, two distinct augmented states (potentially with different $\bsmu$) can generate $\varepsilon$-close observation windows, so no method that uses only the window can guarantee reconstruction error below order $C_{\mathrm{ev}}\alpha(\varepsilon)$ (with $\varepsilon$ set by noise/discretization).
Both viewpoints motivate increasing the window length (larger $k$) and/or the number of sensors (larger $p$) until the empirical mapping becomes sufficiently well-conditioned, as explored in our ablation studies.
\end{remark}

\section{Training Details}
\label{hyper}
To determine the optimal hyperparameter choices for STRIDE-FMMNN in all the examples, we conduct the hyperparameter search using Bayesian optimization \cite{li_hyperband_2018}. The range of hyperparameters considered is listed in Table~\ref{tab:hyperrange}. The details of STRIDE training (including STRIDE-FMMNN and STRIDE-SIREN) and optimized hyperparameters are shown in Table~\ref{tab:detail}. SHRED and SHRED-ROM have the same temporal decoder as STRIDE, and have a shallow decoder with 2 hidden layers of 350 and 400 neurons, respectively, which is the same as the default setting in \cite{tomasetto_reduced_2025}. For all the trained models, we employ an early stopping mechanism: the training terminates once the validation loss does not decrease by more than 0.5\% in 10 epochs over an 80-epoch period.

\begin{table}[htb]
  \caption{Hyperparameter search range for STRIDE-FMMNN training}
  \label{tab:hyperrange}
  \begin{center}
    \begin{small}
      \begin{sc}
        \begin{tabular}{c llll} 
            \toprule
            Hyperparameter & \emph{KS} & \emph{FlowAO} & \emph{SWE (128$\times$128)} & \emph{Seismic} \\
            \midrule
            Number of latent states & 16 -- 128 & 16 -- 128 & 32 -- 256 & 32 -- 256 \\
            Decoder depth           & 2 -- 10 & 2 -- 10 & 3 -- 12 & 3 -- 12 \\
            Decoder rank            & 16 -- 128 & 16 -- 128 & 32 -- 256 & 32 -- 256 \\
            Decoder width           & 256 -- 2048 & 256 -- 2048 & 256 -- 2048 & 256 -- 2048 \\
            Fourier embedding dim   & 0 -- 50 & 0 -- 50 & 0 -- 50 & 0 -- 50 \\
            Learning rate           & 1e-3 -- 5e-3 & 1e-3 -- 5e-3 & 1e-3 -- 5e-3 & 1e-3 -- 4e-3 \\
            Batch size              & 128 -- 1024 & 16 -- 128 & 16 -- 128 & 64 -- 128 \\
            \bottomrule
        \end{tabular}
      \end{sc}
    \end{small}
  \end{center}
\end{table}

\begin{table}[htb]
  \caption{Training details for STRIDE}
  \label{tab:detail}
  \begin{center}
    \begin{small}
      \begin{sc}
        \begin{tabular}{c llll} 
            \toprule
            Hyperparameter & \emph{KS} & \emph{FlowAO} & \emph{SWE (128$\times$128)} & \emph{Seismic} \\
            \midrule
            Lag                     & 50 & 50 & 50 & 50 \\
            Number of sensors       & 2 & 3 & 100 & 100 \\
            Number of space points  & 100 & 5000 & 4096 & 4900 \\
            Number of latent states & 128 & 64 & 128 & 256 \\
            Number of hidden layers & 2 & 2 & 2 & 2 \\
            FMMNN decoder depth     & 3 & 6 & 8 & 4 \\
            FMMNN decoder rank      & 128 & 32 & 32 & 128 \\
            FMMNN decoder width     & 1024 & 256 & 256 & 1024 \\
            SIREN decoder depth     & 3 & 6 & 8 & 4 \\
            SIREN decoder width     & 384 & 96 & 96 & 384 \\
            Fourier embedding dim   & 0 & 5 & 20 & 5 \\
            FMMNN learning rate     & 3e-3 & 3e-3 & 2e-3 & 3e-3 \\
            SIREN learning rate     & 1e-4 & 1e-4 & 1e-4 & 1e-4 \\
            Scheduler gamma/patience & 0.4/10 & 0.4/10 & 0.4/10 & 0.4/10 \\
            Batch size              & 256 & 128 & 128 & 128 \\
            \bottomrule
        \end{tabular}
      \end{sc}
    \end{small}
  \end{center}
\end{table}

\newpage 
Tables~\ref{tab:auxMLP} and~\ref{tab:auxLSTM} specify the training details for the auxiliary models used in parameter estimation (see Fig.~\ref{fig:estimation}) and forecasting (see Fig.~\ref{fig:forecasting}).

\begin{table}[!htb]
  \caption{Training details for auxiliary MLP (parameter estimation)}
  \label{tab:auxMLP}
  \begin{center}
    \begin{small}
      \begin{sc}
        \begin{tabular}{c ll} 
            \toprule
            Hyperparameter & \emph{SWE (128$\times$128)} & \emph{FlowAO} \\
            \midrule
            Latent dimension        & 128 & 64 \\
            Hidden dimension        & 128 & 64 \\
            Number of hidden layers & 2 & 2 \\
            Parameter dimension     & 2 & 1 \\
            Activation function     & GELU & GELU \\
            Learning rate           & 2e-4 & 5e-4 \\
            Number of epochs        & 2000 & 2000 \\
            Batch size              & 1024 & 1024 \\
            \bottomrule
        \end{tabular}
      \end{sc}
    \end{small}
  \end{center}
\end{table}

\begin{table}[!htb]
  \caption{Training details for auxiliary LSTM (forecasting)}
  \label{tab:auxLSTM}
  \begin{center}
    \begin{small}
      \begin{sc}
        \begin{tabular}{c ll} 
            \toprule
            Hyperparameter & \emph{SWE (128$\times$128)} & \emph{Seismic} \\
            \midrule
            Input length            & 20 & 20 \\
            Output length           & 1 & 1 \\
            Latent dimension        & 128 & 256 \\
            Hidden dimension        & 256 & 256 \\
            Number of hidden layers & 2 & 2 \\
            Learning rate           & 1e-3 & 1e-3 \\
            Scheduler gamma/patience& 0.5/10 & 0.5/10 \\
            Number of epochs        & 400 & 400 \\
            Dropout                 & 0.1 & 0.1 \\
            Batch size              & 1024 & 1024 \\
            \bottomrule
        \end{tabular}
      \end{sc}
    \end{small}
  \end{center}
\end{table}

\end{document}